\documentclass[sn-basic]{sn-jnl}% Style for submissions to Nature Portfolio journals

\usepackage{silence}
\usepackage[utf8]{inputenc} % allow utf-8 input
\usepackage[T1]{fontenc}    % use 8-bit T1 fonts
\usepackage{url}            % simple URL typesetting
\usepackage{booktabs}       % professional-quality tables
\usepackage{amsfonts}       % blackboard math symbols
\usepackage{nicefrac}       % compact symbols for 1/2, etc.
\usepackage{microtype}      % microtypography
\usepackage[table]{xcolor}         % colors
\usepackage{algorithm}
\usepackage[noend]{algorithmic}
\usepackage{bm}
\usepackage{enumitem}
\usepackage{comment}
\usepackage{dsfont}
\usepackage{amsmath}
\usepackage{amssymb}
\usepackage{mathtools}
\usepackage{amsthm}
\usepackage[]{amsbsy}
\usepackage{commath}
\usepackage{multirow}
\usepackage{thm-restate} % re-state lemmas, theorems, etc.
\usepackage{physics} % for differential symbols
\usepackage{float} % for the subfloat command
\usepackage[caption=false,font=footnotesize]{subfig}
\usepackage[skip=0pt]{caption}

\usepackage{graphicx}
\graphicspath{{figures/}} 
\theoremstyle{plain}

\newtheorem{lemma}{Lemma}
\newtheorem{corollary}{Corollary}
\newtheorem{assumption}{Assumption}
\theoremstyle{definition}

\theoremstyle{remark}

\usepackage[capitalize,noabbrev]{cleveref}
\crefname{assumption}{Assumption}{Assumptions}
\crefname{theorem}{Theorem}{Theorems}
\crefname{equation}{}{}
\crefname{ALC@unique}{Line}{Lines}

\let\originalleft\left
\let\originalright\right
\renewcommand{\left}{\mathopen{}\mathclose\bgroup\originalleft}
\renewcommand{\right}{\aftergroup\egroup\originalright}

\newcommand{\paren}[1]{\left(#1\right)}
\newcommand{\bracket}[1]{\left[#1\right]}

\newcommand{\Cov}{\mathrm{Cov}}
\newcommand{\given}{\;\middle|\;}

\usepackage{etoolbox}
\newcounter{myalg}
\AtBeginEnvironment{algorithmic}{\refstepcounter{myalg}}
\makeatletter
\newcommand\thefontsize{The current font size is: \f@size pt}
\@addtoreset{ALC@unique}{myalg}
\makeatother

\definecolor{LightGray}{gray}{0.9}

\DeclareMathOperator*{\argmax}{argmax} % thin space, limits underneath in displays

\DeclareMathOperator{\E}{\mathbb{E}}
\DeclareMathOperator{\Prob}{\mathbb{P}}
\DeclareMathOperator{\V}{\mathbb{V}}
\DeclareMathOperator{\R}{\mathbb{R}}

\renewcommand{\paragraph}[1]{\textbf{#1}\hspace{0em}}

\definecolor{mine}{RGB}{205, 232, 248}%

\makeatletter
\def\endthebibliography{%
  \def\@noitemerr{\@latex@warning{Empty `thebibliography' environment}}%
  \endlist
}

\AtBeginEnvironment{appendices}{\crefalias{section}{appendix}}
\AtBeginEnvironment{appendices}{\crefalias{subsection}{appendix}}

\begin{document}

\title{Model-Based Epistemic Variance of Values for Risk-Aware Policy Optimization}

\author*[1,2]{\fnm{Carlos E.} \sur{Luis}}\email{carlos@robot-learning.de}

\author[1,2]{\fnm{Alessandro G.} \sur{Bottero}}\email{alessandrogiacomo.bottero@bosch.com}

\author[1]{\fnm{Julia} \sur{Vinogradska}}\email{julia.vinogradska@bosch.com}

\author[1]{\fnm{Felix} \sur{Berkenkamp}}\email{felix.berkenkamp@bosch.com}

\author[2,3,4,5]{\fnm{Jan} \sur{Peters}}\email{jan.peters@tu-darmstadt.de}

\affil*[1]{\orgdiv{Bosch Corporate Research}, \orgaddress{\street{Robert-Bosch-Campus
1}, \city{Renningen}, \postcode{71272}, \country{Germany}}}

\affil[2]{\orgdiv{Intelligent Autonomous Systems Group}, \orgname{Technical University
Darmstadt}, \orgaddress{\street{Hochschulstr. 10}, \city{Darmstadt},
\postcode{64289}, \country{Germany}}}

\affil[3]{\orgdiv{German Research Center for Artificial Intelligence (DFKI)},
\orgaddress{\street{Landwehrstr. 50A}, \city{Darmstadt}, \postcode{64293},
\country{Germany}}}

\affil[4]{\orgdiv{Hessian.AI}, \orgaddress{\country{Germany}}}

\affil[5]{\orgdiv{Centre for Cognitive Science},
\orgaddress{\city{Darmstadt}, \country{Germany}}}

\abstract{ We consider the problem of quantifying uncertainty over expected cumulative
 rewards in model-based reinforcement learning. In particular, we focus on
 characterizing the \emph{variance} over values induced by a distribution over Markov
 decision processes (MDPs). Previous work upper bounds the posterior variance over
 values by solving a so-called uncertainty Bellman equation (UBE), but the
 over-approximation may result in inefficient exploration. We propose a new UBE whose
 solution converges to the true posterior variance over values and leads to lower regret
 in tabular exploration problems. We identify challenges to apply the UBE theory beyond
 tabular problems and propose a suitable approximation. Based on this approximation, we
 introduce a general-purpose policy optimization algorithm, $Q$-Uncertainty Soft
 Actor-Critic (QU-SAC), that can be applied for either risk-seeking or risk-averse
 policy optimization with minimal changes. Experiments in both online and offline RL
 demonstrate improved performance compared to other uncertainty estimation methods. }   

\keywords{Model-Based Reinforcement Learning, Bayesian Reinforcement Learning, Uncertainty Quantification.}

\maketitle

\section{Introduction}
\label{sec:introduction}
The goal of reinforcement learning (RL) is optimal decision-making in an \emph{a priori}
unknown Markov Decision Process (MDP) \citep{sutton_reinforcement_2018}. The RL agent
obtains rewards from interactions with the MDP and optimality is defined by some utility
function of the accumulated rewards (also known as return). The return is, in general, a
random variable due to two distinct types of uncertainty: \emph{aleatoric}, induced by a
combination of the agent's stochastic action selection and the random MDP state
transitions; and \emph{epistemic}, due to having limited data of the unknown MDP
\citep{kiureghian_aleatory_2009}. Aleatoric uncertainty is irreducible, since it is an
inherent property of the problem, while epistemic uncertainty can be reduced by further
interactions with the MDP. A standard utility in the RL literature is the
\emph{expected} return, known as the value function, which is a \emph{risk-neutral}
objective that averages over the \emph{aleatoric} uncertainty without explicit treatment
of epistemic uncertainty. In this paper, we first present a method to explicitly
estimate epistemic uncertainty around the value function and then argue for
\emph{epistemic risk-aware objectives} as a unified framework for tackling problems in
which risk-neutrality leads to sub-optimal solutions.

We motivate the need for risk-aware objectives with two concrete practical
tasks: online exploration and offline optimization. In online exploration, the MDP's
reward signal is sparse and standard RL algorithms based on maximizing expected return
converge to a suboptimal solution even in simple tasks \citep{raffin_smooth_2021}. In
offline optimization, the RL agent does not interact with the MDP and solely relies on a
dataset with limited support; in this case, standard RL algorithms without additional
regularization are known to diverge \citep{levine_offline_2020}. While each of these
problems have been tackled individually in the past, we propose a unified solution by
quantifying epistemic uncertainty and optimizing a simple risk-aware objective:
risk-seeking to encourage exploration in the absence of a reward signal, or risk-averse
for explicit regularization in offline optimization. The two behaviors are controlled by
a single hyperparameter, thus the same algorithm can be applied to both exploration and
offline problems.

In order to model epistemic uncertainty, we adopt the model-based RL (MBRL)
paradigm, in which the RL agent learns a probabilistic model of the MDP
\citep{sutton_dyna_1991}. For tabular RL problems with finite state-action spaces,
provably efficient RL algorithms leverage the learned model of the MDP to derive
epistemic-uncertainty-based rewards that instill exploratory behaviour
\citep{strehl_analysis_2008,jaksch_near-optimal_2010}. Beyond these tabular RL
approaches, modern deep learning MBRL methods quantify epistemic and aleatoric
uncertainty in the learned MDP dynamics
\citep{depeweg_decomposition_2018,chua_deep_2018} and leverage them to optimize the
policy \citep{curi_efficient_2020}. Still, proper uncertainty quantification of
long-term return predictions remains a challenging trade-off between accuracy and
tractable probabilistic inference \citep{deisenroth_pilco_2011}. Despite these
challenges, it has been shown that quantification of uncertainty around the policy's
value enables risk-awareness, i.e., reasoning about the long-term risk of rolling out a
policy. Promising results have been reported for both risk-seeking
\citep{deisenroth_pilco_2011,fan_model-based_2021} and risk-averse
\citep{zhou_deep_2020,yu_mopo_2020} policy optimization.

Similar to prior work in MBRL, we use a Bayesian approach to characterize
uncertainty in the MDP via a posterior distribution \citep{dearden_model_1999}. This
distributional perspective of the RL environment induces distributions over functions of
interest for solving the RL problem, e.g., the value function. Our perspective differs
from \emph{distributional} RL \citep{bellemare_distributional_2017}, whose main object
of study is the \emph{aleatoric} noise around the \emph{return}. As such, distributional
RL models \emph{aleatoric} uncertainty, whereas our Bayesian MBRL perspective focuses on
the \emph{epistemic} uncertainty arising from finite data of the underlying MDP. 

% Recent work combines Bayesian and distributional RL for uncertainty-aware optimization
% of policies \citep{eriksson_sentinel_2022,moskovitz_tactical_2021}.

In this work, we analyze the \emph{variance} of the distribution over value
functions and design an algorithm to estimate it. Our method relies on dynamic
programming and the well-known Bellman equation \citep{bellman_dynamic_1957}. In
particular, previous work by \citet{odonoghue_uncertainty_2018,zhou_deep_2020} showed
that the dynamic programming solution to a so-called uncertainty Bellman equation (UBE)
is a guaranteed upper-bound on the posterior variance of the value function. Our
theoretical result is a new UBE whose solution is \emph{exactly} the posterior variance
of the value function, thus closing the theoretical gap in previous work. Beyond the
theoretical analysis, which we previously published in \citep{luis_model-based_2023}, in
this work we further identify limitations of naive applications of our theoretical
result with neural networks as function approximators and propose a novel solution. Our
aim is to devise a general algorithm for RL problems where optimizing for risk-neutral
objectives is known to underperform. In particular, we consider two such problems:
exploration, which typically requires risk-seeking behaviour, and offline optimization,
which benefits from risk-averse objectives.

~\\
\paragraph{Our contribution.} We propose a novel MBRL algorithm for continuous
control called $Q$-Uncertainty Soft Actor-Critic (QU-SAC) that can be applied for either
risk-seeking or risk-averse policy optimization with minimal changes. This work is an
extension to the conference paper \citep{luis_model-based_2023} that introduces the core
theory around the UBE and a preliminary continuous control algorithm evaluated in a
limited suite of online RL problems. We extend our prior work by: (1) identifying
limitations of the direct application of the UBE theory under function approximation;
(2) improving the performance of the previous control algorithm in
\citet{luis_model-based_2023} via a new approximation to the UBE solution, which in
addition is easier to implement and less computationally demanding; and (3) conducting
several new experiments on exploration tasks from the DeepMind Control (DMC) suite
\citep{tunyasuvunakool_dm_control_2020} and offline RL tasks from the D4RL benchmark
\citep{fu_d4rl_2020}.

To the best of our knowledge, QU-SAC is the first uncertainty-based MBRL
algorithm that flexibly handles both online exploration and offline conservative
optimization.

\subsection{Related work}
\paragraph{Model-free Bayesian RL.} Model-free approaches to Bayesian RL directly model the
distribution over values, e.g., with normal-gamma priors \citep{dearden_bayesian_1998}, Gaussian
Processes \citep{engel_bayes_2003} or ensembles of neural networks \citep{osband_deep_2016}.
\citet{jorge_inferential_2020} estimates value distributions using a backwards induction framework,
while \citet{metelli_propagating_2019} propagates uncertainty using Wasserstein barycenters.
\citet{fellows_bayesian_2021} showed that, due to bootstrapping, model-free Bayesian methods infer a
posterior over Bellman operators rather than values.

\paragraph{Model-based Bayesian RL.} Model-based Bayesian RL maintains a posterior over
plausible MDPs given the available data, which induces a distribution over values. The
MDP uncertainty is typically represented in the one-step transition model as a
by-product of model-learning. For instance, the well-known PILCO algorithm by
\citet{deisenroth_pilco_2011} learns a Gaussian Process (GP) model of the transition
dynamics and integrates over the model's total uncertainty to obtain the expected
values. In order to scale to high-dimensional continuous control problems,
\citet{chua_deep_2018} proposes PETS, which uses ensembles of probabilistic neural
networks (NNs) to capture both aleatoric and epistemic uncertainty as first proposed by
\citet{lakshminarayanan_simple_2017}. Both approaches propagate model uncertainty during
policy evaluation and improve the policy via greedy exploitation over this
model-generated noise. Dyna-style \citep{sutton_dyna_1991} actor-critic algorithms have
been paired with model-based uncertainty estimates for improved performance in both
online \citep{buckman_sample-efficient_2018,zhou_efficient_2019} and offline
\citep{yu_mopo_2020,kidambi_morel_2020} RL.

\paragraph{Online RL - Optimism.} To balance exploration and exploitation, provably-efficient RL
algorithms based on \emph{optimism in the face of the uncertainty} (OFU)
\citep{auer_logarithmic_2006,jaksch_near-optimal_2010} rely on building upper-confidence
(optimistic) estimates of the true values. These optimistic values correspond to a modified MDP
where the rewards are enlarged by an uncertainty bonus, which encourages exploration. In practice,
however, the aggregation of optimistic rewards may severely over-estimate the true values, rendering
the approach inefficient \citep{osband_why_2017}. \citet{odonoghue_uncertainty_2018} shows that
methods that approximate the variance of the values can result in much tighter upper-confidence
bounds, while \citet{ciosek_better_2019} demonstrates their use in complex continuous control
problems. Similarly, \citet{chen_ucb_2017} proposes a model-free ensemble-based approach to estimate
the variance of values.

\paragraph{Offline RL - Pessimism.} In offline RL, the policy is optimized solely from offline
(static) data rather than from online interactions with the environment \citep{levine_offline_2020}.
A primary challenge in this setting is known as \emph{distribution shift}, which refers to the shift
between the state-action distribution of the offline dataset and that of the learned policy. The
main underlying issue with distribution shifts in offline RL relates to querying value functions
out-of-distribution (OOD) with no opportunity to correct for generalization errors via online
interactions (as in the typical RL setting). One prominent technique to deal with distribution
shifts is known as \emph{conservatism} or \emph{pessimism}, where a pessimistic value function
(typically a lower bound of the true values) is learned by regularizing OOD actions
\citep{kumar_conservative_2020,bai_pessimistic_2022}. Model-based approaches to pessimism can be
sub-divided into uncertainty-free \citep{yu_combo_2021,rigter_rambo-rl_2022} and uncertainty-based
methods \citep{yu_mopo_2020,kidambi_morel_2020,jeong_conservative_2023}. While uncertainty-free
pessimism circumvents the need to explicitly estimate the uncertainty, the current state-of-the-art
method CBOP \citep{jeong_conservative_2023} is uncertainty-based. Our QU-SAC algorithm falls into
the uncertainty-based category and differentiates from prior work over which uncertainty it
estimates: MOPO \citep{yu_mopo_2020} uses the maximum aleatoric standard deviation of a dynamics
ensemble forward prediction, MOREL \citep{kidambi_morel_2020} is similar but uses the maximum
pairwise difference of the mean predictions, CBOP \citep{jeong_conservative_2023} instead does
approximate Bayesian inference directly on the $Q$-value predictions conditioned on empirical
(bootstrapped) return estimates. Instead, QU-SAC learns a Bayesian estimate of the $Q$-values
variance via approximately solving a UBE. To the best of our knowledge, this is the first time a
UBE-based algorithm is used for offline RL.

\paragraph{Unified offline / online RL.} The closest body of work in which
offline and online optimization are treated under the same umbrella is that of
offline-to-online RL, also known as online fine-tuning
\citep{lee_offline--online_2022,nakamoto_cal-ql_2023}. \citet{lei_uni-o4_2024} unifies
the offline and online phases under the same objective function, but the training
procedure between both phases differs, adding further complexity.
\citet{zhao_improving_2023} uses the same base algorithm (SAC) in both phases, but
risk-awareness is procured by different methods: CQL \citep{kumar_conservative_2020} in
the offline phase, and SUNRISE \citep{lee_sunrise_2021} for online fine-tuning.

\paragraph{Uncertainty in RL.} Interest about the higher moments of the \emph{return} of a policy
dates back to the work of \citet{sobel_variance_1982}, showing these quantities obey a Bellman
equation. Methods that leverage these statistics of the return are known as \emph{distributional} RL
\citep{tamar_temporal_2013,bellemare_distributional_2017}. Instead, we focus specifically on
estimating and using the \emph{variance} of the \emph{expected return} for policy optimization. A
key difference between the two perspectives is the type of uncertainty they model: distributional RL
models the \emph{aleatoric} uncertainty about the returns, which originates from the aleatoric noise
of the MDP transitions and the stochastic policy; our perspective studies the \emph{epistemic}
uncertainty about the value function, due to incomplete knowledge of the MDP. Provably efficient RL
algorithms use this isolated epistemic uncertainty as a signal to balance exploring the environment
and exploiting the current knowledge.

\paragraph{UBE-based RL.} \citet{odonoghue_uncertainty_2018} proposes a UBE whose fixed-point
solution converges to a guaranteed upper-bound on the posterior variance of the value function in
the tabular RL setting. This approach was implemented in a model-free fashion using the DQN
\citep{mnih_playing_2013} architecture and showed performance improvements in Atari games. Follow-up
work by \citet{markou_bayesian_2019} empirically shows that the upper-bound is loose and the
resulting over-approximation of the variance impacts negatively the regret in tabular exploration
problems. \citet{zhou_deep_2020} proposes a modified UBE with a tighter upper-bound on the value
function, which is then paired with proximal policy optimization (PPO)
\citep{schulman_proximal_2017} in a conservative on-policy model-based approach to solve
continous-control tasks. Our QU-SAC algorithm integrates UBE-based uncertainty quantification into a
model-based soft actor-critic (SAC) \citep{haarnoja_soft_2018} architecture similar to
\citet{janner_when_2019,froehlich_-policy_2022}. 

\section{Problem Statement}
\label{sec:problem_statement}
We consider an agent that acts in an infinite-horizon MDP $\mathcal{M} = \set{\mathcal{S},
\mathcal{A}, p, \rho, r, \gamma}$ with finite state space $\abs{\mathcal{S}} = S$, finite action
space $\abs{\mathcal{A}} = A$, unknown transition function $p: \mathcal{S} \times \mathcal{A} \to
\Delta(S)$ that maps states and actions to the $S$-dimensional probability simplex, an initial state
distribution $\rho: \mathcal{S} \to [0,1]$, a known and bounded reward function $r: \mathcal{S}
\times \mathcal{A} \to \R$, and a discount factor $\gamma \in [0,1)$. Although we consider a known
reward function, the main theoretical results can be easily extended to the case where it is learned
alongside the transition function (see \cref{app:unknown_rewards}). The one-step dynamics $p(s' \mid
s,a)$ denote the probability of going from state $s$ to state $s'$ after taking action $a$. In
general, the agent selects actions from a stochastic policy $\pi: \mathcal{S} \to \Delta(A)$ that
defines the conditional probability distribution $\pi(a\mid s)$. At each time step, the agent is in
some state $s$, selects an action $a \sim \pi(\cdot \mid s)$, receives a reward $r(s,a)$, and
transitions to a next state $s' \sim p(\cdot \mid s,a)$. We define the value function $V^{\pi,p}:
\mathcal{S} \to \R$ of a policy $\pi$ and transition function $p$ as the expected sum of discounted
rewards under the MDP dynamics,
\begin{equation}
  V^{\pi, p}(s) = \E_{\tau \sim P}\bracket{\sum\nolimits_{h=0}^{\infty}\gamma^hr(s_h,a_h) \given s_0 = s},
\end{equation}
where the expectation is taken under the random trajectories $\tau$ drawn from the trajectory
distribution $P(\tau) = \prod_{h=0}^{\infty} \pi(a_h \mid s_h)p(s_{h+1}\mid s_h, a_h)$.

We consider a Bayesian setting similar to previous work by
\citet{odonoghue_uncertainty_2018,odonoghue_variational_2021,zhou_deep_2020}, in which the
transition function $p$ is a random variable with some known prior distribution $\Phi(p)$. As the
agent interacts in $\mathcal{M}$, it collects data\footnote{We omit time-step subscripts and refer
to dataset $\mathcal{D}$ as the collection of all available transition data.} $\mathcal{D}$ and
updates its posterior belief $\Phi(p \mid \mathcal{D})$ via Bayes' rule. In what follows, we omit
further qualifications and refer to $\Phi$ as the posterior over transition functions. Such
distribution over transition functions naturally induces a distribution over value functions. The
main focus of this paper is to study methods that estimate the \emph{variance} of the value function
$V^{\pi,p}$ under $\Phi$, namely $\V_{p \sim \Phi} \bracket{V^{\pi, p}(s)}$. Our theoretical results
extend to state-action value functions (see \cref{app:extension_q_values}). The motivation behind
studying this quantity is its potential for risk-aware optimization.

A method to estimate an upper-bound the variance of $Q$-values by solving a UBE was introduced by
\citet{zhou_deep_2020}. Their theory holds for a class of MDPs where the value functions and
transition functions are uncorrelated. This family of MDPs is characterized by the following
assumptions:
\begin{assumption}[Parameter Independence \citep{dearden_model_1999}]
    \label{assumption:transitions}
    The posterior over the random vector $p(\cdot \mid s, a)$ is independent for each pair $(s,
    a) \in \mathcal{S} \times \mathcal{A}$.
\end{assumption}
\begin{assumption}[Acyclic MDP \citep{odonoghue_uncertainty_2018}]
    \label{assumption:acyclic}
    For any realization of $p$, the MDP $\mathcal{M}$ is a directed acyclic graph, i.e., states are
    not visited more than once in any given episode.
\end{assumption}
\Cref{assumption:transitions} is satisfied when modelling state transitions as independent
categorical random variables for every pair $(s, a)$, with the unknown parameter vector $p(\cdot
\mid s, a)$ under a Dirichlet prior \citep{dearden_model_1999}. \Cref{assumption:acyclic} is
non-restrictive as any finite-horizon MDP with cycles can be transformed into an equivalent
time-inhomogeneous MDP without cycles by adding a time-step variable $h$ to the state-space. Since
the state-space is finite-dimensional, for infinite-horizon problems we consider the existence of a
terminal (absorbing) state that is reached within a finite number of steps. The direct consequence
of these assumptions is that the random variables $V^{\pi,p}(s')$ and $p(s' \mid s,a)$ are
independent (see \cref{lemma:independence_from_assumptions,lemma:uncorrelated_property} in
\cref{app:proofshm_ube} for a formal proof). 

Other quantities of interest are the posterior mean transition function starting from the current
state-action pair $(s,a)$, 
\begin{equation}
  \label{eq:mean_model}
  \bar{p}(\cdot \mid s,a) = \E_{p \sim \Phi}\bracket{p(\cdot \mid s,a)}, 
\end{equation}
and the posterior mean value function for any $s \in \mathcal{S}$,
\begin{equation}
  \label{eq:mean_value_function}
  \bar{V}^{\pi}(s) = \E_{p \sim \Phi}\bracket{V^{\pi,p}(s)}.
\end{equation}
Note that $\bar{p}$ is a transition function that combines both aleatoric \emph{and} epistemic
uncertainty. Even if we limit the posterior $\Phi$ to only include deterministic transition
functions, $\bar{p}$ remains a stochastic transition function due to the epistemic uncertainty.

In \citet{zhou_deep_2020}, \emph{local} uncertainty is defined as
\begin{equation}
  \label{eq:pombu_rewards}
  w(s) = \V_{p\sim \Phi}
  \bracket{\sum\nolimits_{a, s'}\pi(a \mid s) p(s' \mid s,a) \bar{V}^\pi(s')},
\end{equation}
which captures variability of the posterior mean value function at the next state $s'$. Based on
this local uncertainty, \citet{zhou_deep_2020} proposes the UBE
\begin{equation}
  \label{eq:ube_pombu}
  W^\pi(s) = 
  \gamma ^ 2w(s) + \gamma^2\sum_{a, s'}\pi(a \mid s)\bar{p}(s' \mid s,a) W^\pi(s'),
\end{equation}
that propagates the local uncertainty using the posterior mean dynamics. It was proven that the
fixed-point solution of \cref{eq:ube_pombu} is an upper-bound of the epistemic variance of the
values, i.e., it satisfies $W^\pi(s) \geq \V_{p \sim \Phi} \bracket{V^{\pi, p}(s)}$ for all $s$.

\section{Uncertainty Bellman Equation}
In this section, we build a new UBE whose fixed-point solution is \emph{equal} to the variance of
the value function and we show explicitly the gap between \cref{eq:ube_pombu} and $\V_{p \sim
\Phi}\bracket{V^{\pi, p}(s)}$.

The values $V^{\pi,p}$ are the fixed-point solution to the Bellman expectation equation, which
relates the value of the current state $s$ with the value of the next state $s'$. Further, under
\cref{assumption:transitions,assumption:acyclic}, applying the expectation operator to the Bellman
recursion results in $\bar{V}^{\pi}(s) = V^{\pi, \bar{p}}(s)$. The Bellman recursion propagates
knowledge about the \emph{local} rewards $r(s,a)$ over multiple steps, so that the value function
encodes the \emph{long-term} value of states if we follow policy $\pi$. Similarly, a UBE is a
recursive formula that propagates a notion of \emph{local uncertainty}, $u(s)$, over multiple steps.
The fixed-point solution to the UBE, which we call the $U$-values, encodes the \emph{long-term
epistemic uncertainty} about the values of a given state.

Previous formulations by \citet{odonoghue_uncertainty_2018,zhou_deep_2020} differ only on their
definition of the local uncertainty and result on $U$-values that upper-bound the posterior
variance of the values. The first key insight of our paper is that we can define $u$ such that the
$U$-values converge exactly to the variance of values. This result is summarized in the following
theorem:
\begin{restatable}{theorem}{ube}
  \label{thm:ube}
  Under \cref{assumption:transitions,assumption:acyclic}, for any $s \in \mathcal{S}$ and policy
  $\pi$, the posterior variance of the value function, $U^\pi = \V_{p \sim \Phi}
  \bracket{V^{\pi,p}}$ obeys the uncertainty Bellman equation
  \begin{equation}
    \label{eq:bellman_exact}
    U^\pi(s) = 
    \gamma ^ 2u(s) + \gamma^2\sum_{a, s'}\pi(a \mid s)\bar{p}(s' \mid s,a) U^\pi(s'),
  \end{equation}
  where $u(s)$ is the local uncertainty defined as
  \begin{equation}
    \label{eq:bellman_exact_reward}
    u(s) = \V_{a, s' \sim \pi, \bar{p}} \bracket{\bar{V}^\pi(s')} -
    \E_{p \sim \Phi} \bracket{\V_{a, s' \sim \pi, p} \bracket{V^{\pi, p}(s')}}.
  \end{equation}
\end{restatable}
\begin{proof}
  See \cref{app:proofshm_ube}.
\end{proof}
One may interpret the $U$-values from \cref{thm:ube} as the associated state-values of an alternate
\emph{uncertainty MDP}, $\mathcal{U} = \set{\mathcal{S}, \mathcal{A}, \bar{p}, \rho, \gamma
^2u, \gamma ^2}$, where the agent receives uncertainty rewards and transitions according to the
mean dynamics $\bar{p}$.

A key difference between $u$ and $w$ is how they represent epistemic uncertainty: in the former,
it appears only within the first term, through the one-step variance over $\bar{p}$; in the
latter, the variance is computed over $\Phi$. While the two perspectives may seem fundamentally
different, in the following theorem we present a clear relationship that connects \cref{thm:ube}
with the upper bound \cref{eq:ube_pombu}.
\begin{restatable}{theorem}{connections}
  \label{thm:connection_uncertainties}
  Under \cref{assumption:transitions,assumption:acyclic}, for any $s \in \mathcal{S}$ and policy
  $\pi$, it holds that $u(s) = w(s) - g(s)$, where $g(s) = \E_{p \sim
  \Phi}\bracket{\V_{a,s' \sim \pi, p} \bracket{V^{\pi,p}(s')}- \V_{a,s' \sim \pi, p}
  \bracket{\bar{V}^\pi(s')}}$. Furthermore, we have that the gap $g(s)$ is non-negative, thus $u(s) \leq w(s)$.
\end{restatable}
\begin{proof}
  See \cref{app:proofshm_connections}.
\end{proof}

The gap $g(s)$ of \cref{thm:connection_uncertainties} can be interpreted as
the \emph{average difference} of aleatoric uncertainty about the next values with respect to the
mean values. The gap vanishes only if the epistemic uncertainty goes to zero, or if the MDP and
policy are both deterministic.

We directly connect \cref{thm:ube,thm:connection_uncertainties} via the equality
\begin{equation}
  \label{eq:interpretation}
	\underbrace{\V_{a, s' \sim \pi, \bar{p}} \bracket{\bar{V}^\pi(s')}}_\textrm{total} = \underbrace{w(s)}_\textrm{epistemic} + \underbrace{\E_{p \sim \Phi}\bracket{\V_{a,s' \sim \pi, p} \bracket{\bar{V}^\pi(s')}}}_\textrm{aleatoric},
\end{equation}
which helps us analyze our theoretical results. The uncertainty reward defined in
\cref{eq:bellman_exact_reward} has two components: the first term corresponds to the \emph{total
uncertainty} about the \emph{mean} values of the next state, which is further decomposed in
\cref{eq:interpretation} into an epistemic and aleatoric components. When the epistemic uncertainty
about the MDP vanishes, then $w(s) \to 0$ and only the aleatoric component remains. Similarly,
when the MDP and policy are both deterministic, the aleatoric uncertainty vanishes and we have
$\V_{a, s' \sim \pi, \bar{p}} \bracket{\bar{V}^\pi(s')} = w(s)$. The second term of
\cref{eq:bellman_exact_reward} is the \emph{average aleatoric uncertainty} about the value of the
next state. When there is no epistemic uncertainty, this term is non-zero and exactly equal to the
alectoric term in \cref{eq:interpretation} which means that $u(s) \to 0$. Thus, we can interpret
$u(s)$ as a \emph{relative} local uncertainty that subtracts the average aleatoric noise out of
the total uncertainty around the mean values. Perhaps surprisingly, our theory allows negative
$u(s)$ (see \cref{subsec:toy_example} for a concrete example).

Through \cref{thm:connection_uncertainties} we provide an alternative proof of why the UBE
\cref{eq:ube_pombu} results in an upper-bound of the variance, specified by the next corollary.
\begin{corollary}
  \label{cor:pombu}
  Under \cref{assumption:transitions,assumption:acyclic}, for any $s \in \mathcal{S}$ and policy
  $\pi$, it holds that the solution to the uncertainty Bellman equation \cref{eq:ube_pombu}
  satisfies $W^\pi(s) \geq U^\pi(s)$.
\end{corollary}
\begin{proof}
  The solution to the Bellman equations \cref{eq:bellman_exact,eq:ube_pombu} are the value functions
  under some policy $\pi$ of identical MDPs except for their reward functions. Given two identical
  MDPs $\mathcal{M}_1$ and $\mathcal{M}_2$ differing only on their corresponding reward functions
  $r_1$ and $r_2$, if $r_1 \leq r_2$ for any input value, then for any trajectory $\tau$ we have
  that the returns (sum of discounted rewards) must obey $R_1(\tau) \leq R_2(\tau)$. Lastly, since
  the value functions $V_1^\pi$, $V_2^\pi$ are defined as the expected returns under the same
  trajectory distribution, and the expectation operator preserves inequalities, then we
  have that $R_1(\tau) \leq R_2(\tau) \implies  V_1^\pi \leq V_2^\pi$.
\end{proof}
\Cref{cor:pombu} reaches the same conclusions as \citet{zhou_deep_2020}, but it brings important
explanations about their upper bound on the variance of the value function. First, by
\cref{thm:connection_uncertainties} the upper bound is a consequence of the over approximation of
the reward function used to solve the UBE. Second, the gap between the exact reward function
$u(s)$ and the approximation $w(s)$ is fully characterized by $g(s)$ and brings interesting
insights. In particular, the influence of the gap term depends on the stochasticity of the dynamics
and the policy. In the limit, the term vanishes under deterministic transitions and action
selection. In this scenario, the upper-bound found by \citet{zhou_deep_2020} becomes tight.

Our method returns the exact \emph{epistemic} uncertainty about the values by considering the
inherent aleatoric uncertainty of the MDP and the policy. In a practical RL setting, disentangling
the two sources of uncertainty is key for effective exploration. We are interested in exploring
regions of high epistemic uncertainty, where new knowledge can be obtained. If the variance estimate
fuses both sources of uncertainty, then we may be guided to regions of high uncertainty but with
little information to be gained.

\begin{figure}[t]
  \centering
  \includegraphics[width=0.7\columnwidth]{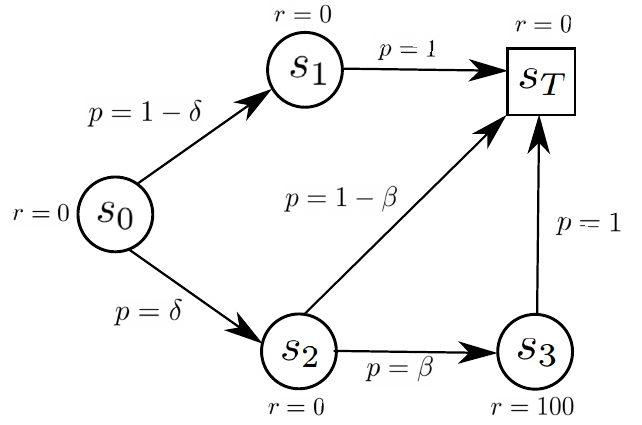}
  \caption{Toy example Markov Reward Process. The random variables $\delta$ and $\beta$ indicate epistemic uncertainty about the MRP. State $s$ is an absorbing (terminal) state.}
  \label{fig:toy_mdp} 
\end{figure}

\subsection{Toy Example}
\label{subsec:toy_example}
To illustrate the theoretical findings of this paper, consider the simple Markov reward process
(MRP) of \cref{fig:toy_mdp}. Assume $\delta$ and $\beta$ to be random variables drawn from a
discrete uniform distribution $\delta \sim \text{Unif}(\set{0.7, 0.6})$ and $\beta \sim
\text{Unif}(\set{0.5, 0.4})$. As such, the distribution over possible MRPs is finite and composed of
the four possible combinations of $\delta$ and $\beta$. Note that the example satisfies
\cref{assumption:transitions,assumption:acyclic}. In \cref{tab:uncertainty_estimationoy_example}
we include the results for the uncertainty rewards and solution to the respective UBEs (the results
for $s_1$ and $s_3$ are trivially zero). For state $s_2$, the upper-bound $W^\pi$ is tight and we
have $W^\pi(s_2) = U^\pi(s_2)$. In this case, the gap vanishes not because of lack of stochasticity,
but rather due to lack of epistemic uncertainty about the next-state values. Indeed, the values for
$s_3$ and $s$ are independent of $\delta$ and $\beta$, which results in the gap terms for $s_2$
cancelling out. For state $s_0$ the gap is non-zero and $W^\pi$ overestimates the variance of the
value by $\sim36\%$. Our UBE formulation prescribes a \emph{negative} reward to be propagated in
order to obtain the correct posterior variance.

\begin{table}[t]
\caption{Comparison of local uncertainty rewards and solutions to the UBE associated with the toy
example from \Cref{fig:toy_mdp}. The $U$-values converge to the true posterior variance of the values, while $W^\pi$ obtains an upper-bound.} \label{tab:uncertainty_estimationoy_example}
\begin{tabular}{c|c|c|c|c}
\textbf{States}  & $u(s)$ & $w(s)$ & $W^\pi(s)$ & $U^\pi(s)$ \\
\hline
$s_0$  & $-0.6$ & $5.0$ & $21.3$ & $15.7$ \\
$s_2$  & $25.0$ & $25.0$ & $25.0$ & $25.0$ \\
\end{tabular}
\end{table}

\section{Uncertainty-Aware Policy Optimization}
\label{sec:optimistic_exploration}
In this section, we propose techniques to leverage uncertainty quantification of $Q$-values for both
online and offline RL problems. In what follows, we consider the general setting with unknown
rewards and define $\Gamma$ to be the posterior distribution over MDPs, from which we can sample
both reward and transition functions. Define $\hat{U}^\pi$ to be an estimate of the posterior
variance over $Q$-values for some policy $\pi$. Then, we consider algorithms that perform policy
updates via the following upper (or lower) confidence bound \citep{auer_logarithmic_2006} type of
optimization problem
\begin{equation}
  \label{eq:policy_opt}
  \pi = \argmax\nolimits_\pi \bar{Q}^\pi + \lambda \sqrt{\hat{U}^\pi},
\end{equation}
where $\bar{Q}^\pi$ is the posterior mean value function and $\lambda$ is a risk-awareness
parameter. A positive $\lambda$ corresponds to risk-seeking, optimistic exploration while negative
$\lambda$ denotes risk-averse, pessimistic anti-exploration.

\cref{algorithm:our_algorithm} describes our general framework to estimate $\bar{Q}^\pi$ and
$\hat{U}^\pi$: we sample an ensemble of $N$ MDPs from the current posterior $\Gamma$ in
\cref{line:sample_model} and use it to solve the Bellman expectation equation in
\cref{line:q_function}, resulting in an ensemble of $N$ corresponding $Q$ functions and the
posterior mean $\bar{Q}^\pi$. Lastly, $\hat{U}^\pi$ is estimated in \cref{line:q_variance} via a
generic variance estimation method \texttt{qvariance}. In what follows, we provide concrete
implementations of \texttt{qvariance} both in tabular and continuous problems.

\begin{algorithm}[t]
   \caption{Model-based $Q$-variance estimation}
   \label{algorithm:our_algorithm}
\begin{algorithmic}[1]
  \STATE {\bfseries Input:} Posterior MDP $\Gamma$, policy $\pi$.

  \STATE $\set{p_i, r_i}_{i=1}^{N}$ $\leftarrow$ \texttt{sample\textunderscore mdp}$(\Gamma)$
  \label{line:sample_model}

  \STATE $\bar{Q}^\pi$, $\set{Q_i}_{i=1}^{N}$ $\leftarrow$\texttt{solve\textunderscore
  bellman}$\paren{\set{p_i, r_i}_{i=1}^{N}, \pi}$ \label{line:q_function}

  \STATE $\hat{U}^\pi$ $\leftarrow$ \texttt{qvariance}$\paren{\set{p_i, r_i, Q_i}_{i=1}^{N}, \bar{Q}^\pi, \pi}$ \label{line:q_variance}
\end{algorithmic}
\end{algorithm}

\subsection{Tabular Problems}
\label{subsec:tabular_implementation}
For problems with tabular representations of the state-action space, we implement \texttt{qvariance}
by directly solving the proposed UBE\footnote{For the UBE-based methods we use the equivalent
equations for $Q$-functions, see \cref{app:q_uncertainty_rewards} for
details.}\cref{eq:bellman_exact}, which we denote \texttt{exact-ube}. For this purpose, we impose a
Dirichlet prior on the transition function and a standard Normal prior for the rewards
\citep{odonoghue_making_2019}, which leads to closed-form posterior updates. After sampling $N$
times from the MDP posterior (\cref{line:sample_model}), we obtain the $Q$-functions
(\cref{line:q_function}) in closed-form by solving the corresponding Bellman equation. The
uncertainty rewards are estimated via sample-based approximations of the expectations/variances
therein. Lastly, we solve \cref{eq:policy_opt} via policy iteration until convergence is achieved or
until a maximum number of steps is reached.

~\\
\paragraph{Practical bound.} The choice of a Dirichlet prior violates \cref{assumption:acyclic}. A
challenge arises in this practical setting: \texttt{exact-ube} may result in \emph{negative}
$U$-values, as a combination of (\textit{i}) the assumptions not holding and (\textit{ii}) the
possibility of negative uncertainty rewards. While (\textit{i}) cannot be easily resolved, we
propose a practical upper-bound on the solution of \cref{eq:bellman_exact} such that the resulting
$U$-values are non-negative and hence interpretable as variance estimates. We consider the clipped
uncertainty rewards $\tilde{u} = \max(u_{\min}, u(s))$ with corresponding $U$-values
$\tilde{U}^\pi$. It is straightforward to prove that, if $u_{\min} = 0$, then $W^\pi(s) \geq
\tilde{U}^\pi(s) \geq U^\pi(s)$, which means that using $\tilde{U}^\pi$ still results in a
tighter upper-bound on the variance than $W^\pi$, while preventing non-positive solutions to the
UBE. In what follows, we drop this notation and assume all $U$-values are computed from clipped
uncertainty rewards.

\begin{figure*}[t]
	\centering
  \includegraphics[width=1.0\textwidth]{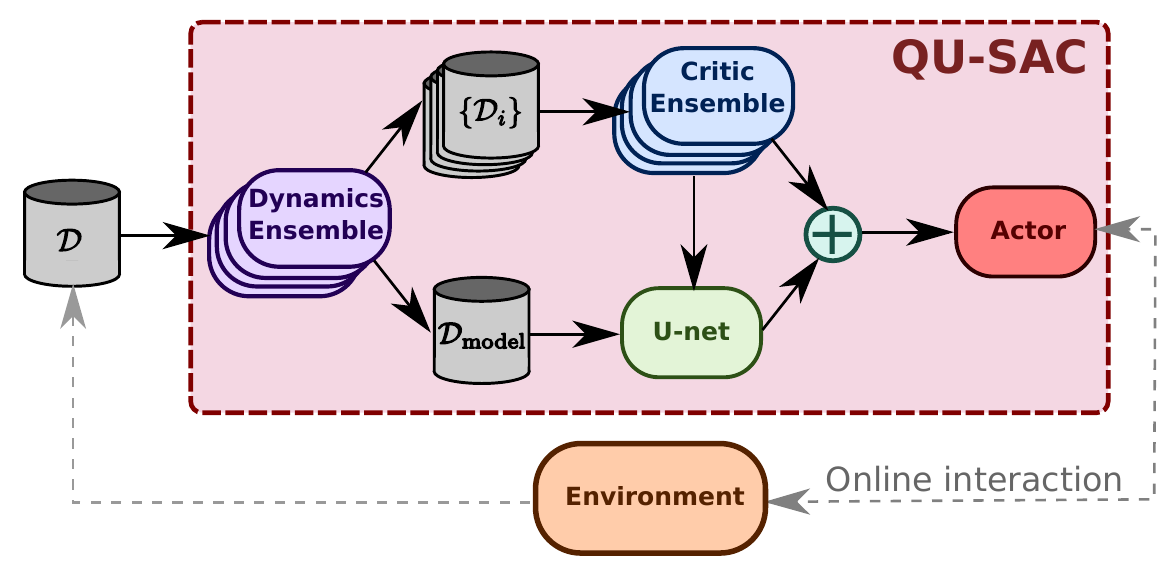}
  \caption{Architecture for $Q$-Uncertainty Soft Actor-Critic (QU-SAC). The dataset $\mathcal{D}$
  may be either static, as in offline RL, or be dynamically populated with online interactions. This
  dataset is used to train an ensemble of dynamics models which is then used for synthetic rollout
  generation. Each member of the ensemble populates its own buffer $\mathcal{D}_i$, which is used to
  train a corresponding ensemble of critics. Additionally, member-randomized rollouts are stored in
  $\mathcal{D}_{\text{model}}$ and used to train a $U$-net, which outputs an estimated epistemic
  variance of the value prediction. Lastly, the actor aims to optimize the risk-aware objective
  \cref{eq:policy_opt}, which combines the output of the critic ensemble and the $U$-net.}
  \label{fig:deep_rl_architecture}
\end{figure*}

\subsection{Continuous Problems}
We tackle problems in continuous domains using neural networks for function approximation. The
resulting architecture is named $Q$-Uncertainty Soft Actor-Critic (QU-SAC) which builds upon MBPO by
\citet{janner_when_2019} and is depicted in \cref{fig:deep_rl_architecture}.

~\\
\paragraph{Posterior dynamics.} In contrast to the tabular implementation, maintaining an explicit
distribution over MDPs from which we can sample is intractable. Instead, we approximate $\Gamma$
with an ensemble, which have been linked to approximate posterior inference
\citep{osband_randomized_2018}. More concretely, we model $\Gamma$ as a discrete uniform
distribution of $N$ probabilistic neural networks, denoted $p_\theta$, that output the mean and
covariance of a Gaussian distribution over next states and rewards \citep{chua_deep_2018}. In this
case, the output of \cref{line:sample_model} in \cref{algorithm:our_algorithm} is precisely the
ensemble of neural networks.

~\\
\paragraph{Critics.} The original MBPO trains $Q$-functions represented as neural networks via
TD-learning on data generated via \emph{model-randomized} $k$-step rollouts from initial states that
are sampled from $\mathcal{D}$. Each forward prediction of the rollout comes from a randomly
selected model of the ensemble and the transitions are stored in a single replay buffer
$\mathcal{D}_{\text{model}}$, which is then fed into a model-free optimizer like SAC.
\Cref{algorithm:our_algorithm} requires a few modifications from the MBPO methodology. To implement
\cref{line:q_function}, in addition to $\mathcal{D}_{\text{model}}$, we create $N$ new buffers
$\set{\mathcal{D}_i}_{i=1}^{N}$ filled with \emph{model-consistent} rollouts, where each $k$-step
rollout is generated under a single model of the ensemble, starting from initial states sampled from
$\mathcal{D}$. We train an ensemble of $N$ value functions $\set{Q_i}_{i=1}^{N}$, parameterized by
$\set{\psi_i}_{i=1}^{N}$, and minimize the residual Bellman error with entropy regularization 
\begin{equation}
  \label{eq:loss_q}
  \mathcal{L}(\psi_i) = \E_{(s,a,r,s') \sim \mathcal{D}_i} \bracket{\paren{y_i- Q_i(s, a; \psi_i)}^2}, 
\end{equation}
where $y_i = r + \gamma \paren{Q_i(s', a'; \bar{\psi}_i) - \alpha \log \pi_\phi(a' \mid s')}$
and $\bar{\psi}_i$ are the target network parameters updated via Polyak averaging for
stability during training \citep{mnih_playing_2013}. The mean $Q$-values, $\bar{Q}^\pi$, are
estimated as the average value of the $Q$-ensemble. 

~\\
\paragraph{Uncertainty rewards.} Our theory prescribes propagating the uncertainty rewards
\cref{eq:bellman_exact_reward} to obtain the \texttt{exact-ube} estimate. It is possible to
approximate these rewards, as in the tabular case, by considering the ensemble of critics as samples
from the value distribution. If we focus only on estimating the positive component of the
\texttt{exact-ube} estimate, i.e., the local uncertainty defined by \citet{zhou_deep_2020}, then a
sample-based approximation is given by
\begin{equation}
  \hat{w}(s, a) = \V_i\bracket{\set{\bar{Q}(s'_i, a'_i)}_{i=1}^N},
\end{equation}
where $s'_i \sim p_i(\cdot \mid s, a)$. While this approach is sensible from our theory
perspective and has lead to promising results in our previous work
\citep{luis_model-based_2023}, it has two main shortcomings in practice: (\textit{i}) it
can be computationally intensive to estimate the rewards and (\textit{ii}) the magnitude
of the rewards is typically low, even if the individual critics have largely different
estimated values. The latter point is illustrated in
\cref{fig:uncertainty_reward_problem}: the term $\hat{w}(s,a)$ captures the local
variance of the average value function, which would be small if the function is
relatively flat around $(s,a)$ or if the dynamics model ensemble yields similar forward
predictions starting from $(s,a)$. Empirically, we found that across many environments
the average magnitude of $\hat{w}(s,a)$ is indeed small (e.g., $\sim10^{-3}$), which
makes training a $U$-net challenging. We alleviate both shortcomings via a simple proxy
uncertainty reward:
\begin{equation}
  \label{eq:modified_rewards}
  \hat{w}_{\text{ub}}(s, a) = \V_i\bracket{\set{Q_i(s,a)}_{i=1}^N},
\end{equation}
which is the sample-based approximation of the value variance. We denote this estimate
\texttt{upper-bound} (thus, the subscript ``ub'' in \cref{eq:modified_rewards}), since in the limit
of infinite samples from the value distribution, solving a UBE with rewards $\hat{w}_{\text{ub}}(s,
a)$ results in an upper bound on the value variance at $(s,a)$.

\begin{figure*}[t]
	\centering
  \includegraphics[width=1.0\textwidth]{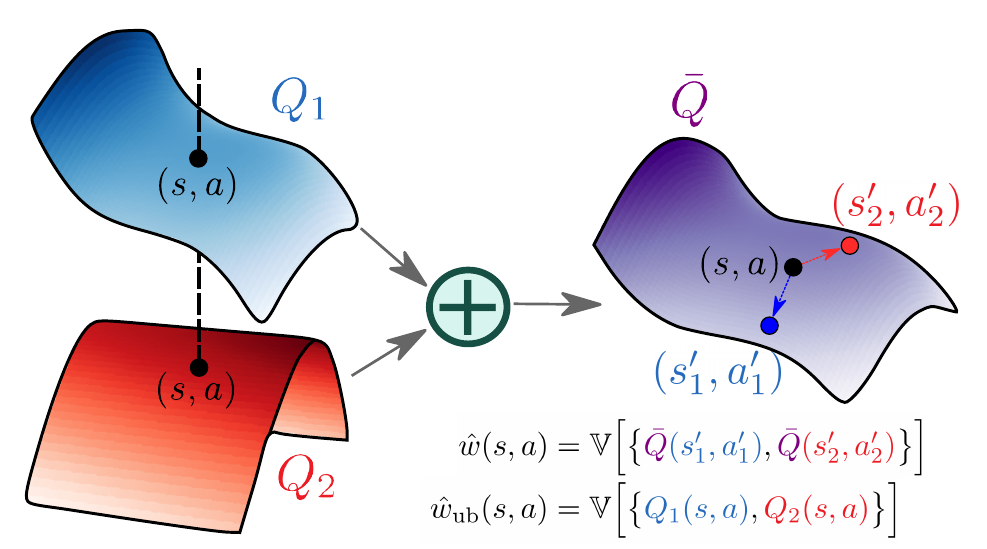}
  \caption{Illustrative example of uncertainty rewards. \textbf{(Left)} ensemble of two value
  functions $\set{Q_1, Q_2}$. \textbf{(Right)} corresponding mean value function $\bar{Q}$. The
  theory prescribes estimating the term in \cref{eq:pombu_rewards}, denoted $\hat{w}(s,a)$, which
  captures local variability of $\bar{Q}$ around $(s,a)$. Empirically, $\hat{w}(s,a)$ can be small
  despite large differences in individual members of the value ensemble, e.g., because $\bar{Q}$ is
  relatively flat around $(s,a)$. We propose the proxy uncertainty reward $\hat{w}_{\text{ub}}(s,
  a)$ which directly captures variability across the value ensemble and is less computationally
  expensive (no dynamics model forward pass).}
  \label{fig:uncertainty_reward_problem}
\end{figure*}

The proxy rewards $\hat{w}_{\text{ub}}(s, a)$ capture explicitly the value ensemble disagreement
rather than local variations of the average value, which empirically results in larger rewards being
propagated through the UBE. Moreover, the proxy reward calculation requires only one forward pass
through the critic ensemble, without need for forward predictions with the dynamics model as for
$\hat{w}(s,a)$.

~\\
\paragraph{Variance estimate.} Similar to critic training, we model the variance estimate
$\hat{U}^\pi$ with a neural network, denoted $U$-net, parameterized by $\varphi$ and trained to
minimize the UBE residual
\begin{equation}
  \label{eq:loss_u}
  \mathcal{L}(\varphi) = \E_{(s,a,r,s') \sim \mathcal{D_{\text{model}}}} \bracket{\paren{z - U(s, a; \varphi)}^2},
\end{equation}
with targets $z = \gamma^2\hat{w}_{\text{ub}}(s, a)  + \gamma ^2 U(s', a'; \bar{\varphi})$ and
target parameters $\bar{\varphi}$ updated like in regular critics. Since we interpret the output of
the network as predictive variances, we use a \emph{softplus} output layer to guarantee non-negative
values. Moreover, we apply a symlog transformation to the UBE targets $z$, as proposed by
\citet{hafner_mastering_2023}, which helps the $U$-net converge to the target values more easily.
Namely, the $U$-net is trained to predict the symlog transform of the target values $z$, defined as
$\text{symlog}(z) = \text{sign}(z)\log(\abs{z} + 1)$. To retrieve the $U$-values, we apply the
inverse transform $\text{symexp}(z) = \text{sign}(z)(\exp(\abs{z}) - 1)$ to the output of the
$U$-net.

~\\
\paragraph{Actor.} The stochastic policy is represented as a neural network with parameters $\phi$,
denoted by $\pi_\phi$. The policy's objective is derived from SAC, where in addition to entropy
regularization, we include the predicted standard deviation of values for uncertainty-aware
optimization.
\begin{equation}
  \label{eq:actor_loss}
    \mathcal{L}(\phi) = \E_{s \sim \mathcal{D_{\text{model}}}} \bracket{\E_{a \sim
    \pi_\phi}\bracket{\bar{Q}(s,a) + \lambda \sqrt{U(s,a)} - \alpha \log \pi_\phi(a \mid s)}}.
\end{equation}

~\\
\paragraph{Online vs offline optimization.} With QU-SAC we aim to use largely the same algorithm to
tackle both online and offline problems. Beyond differences in hyperparameters, the only algorithmic
change in QU-SAC is that for offline optimization we modify the data used to train the actor, critic
and $U$-net to also include data from the offline dataset (an even 50/50 split between offline and
model-generated data in our case), which is a standard practice in offline model-based RL
\citep{rigter_rambo-rl_2022,yu_mopo_2020,yu_combo_2021,jeong_conservative_2023}.

\section{Experiments}
\label{sec:experiments}
In this section, we empirically evaluate the performance of our risk-aware policy optimization
scheme \cref{eq:policy_opt} in various problems and compare against related baselines.

\subsection{Baselines}
In \cref{algorithm:our_algorithm}, we consider different implementations of the \texttt{qvariance}
method to estimate $\hat{U}(s,a)$: \texttt{ensemble-var} directly uses the sample-based
approximation $\hat{w}_{\text{ub}}(s,a)$ in \cref{eq:modified_rewards}; \texttt{pombu} uses the
solution to the UBE \cref{eq:ube_pombu}; \texttt{exact-ube} uses the solution to our proposed UBE
\cref{eq:bellman_exact}; and \texttt{upper-bound} refers to the solution of the UBE with the
modified rewards \cref{eq:modified_rewards}. We also compare against not using any form of
uncertainty quantification, which we refer to as \texttt{ensemble-mean}.

Additionally, in tabular problems we include PSRL by \citet{osband_more_2013} as a baseline since it
typically outperforms recent OFU-based methods
\citep{odonoghue_variational_2021,tiapkin_dirichlet_2022}. We also include MBPO
\citep{janner_when_2019} and MOPO \citep{yu_mopo_2020} as baselines for online and offline problems,
respectively.

\subsection{Gridworld Exploration Benchmark}
We evaluate the tabular implementation in grid-world environments where exploration is key to find
the optimal policy.

\begin{figure*}
	\centering
  \includegraphics[width=\textwidth]{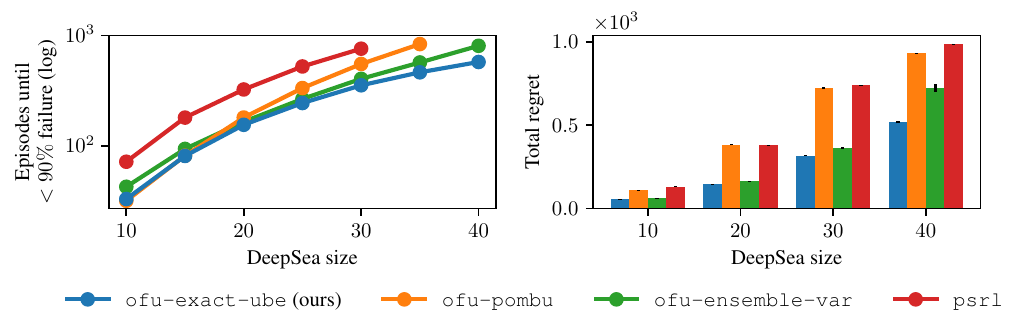}
  \caption{Performance in the \emph{DeepSea} benchmark. Lower values in plots indicate better
  performance. (Left) Learning time is measured as the first episode where the sparse reward has
  been found at least in 10\% of episodes so far. (Right) Total regret is approximately equal to the
  number of episodes where the sparse reward was not found. Results represent the average over 5
  random seeds, and vertical bars on total regret indicate the standard error. Our variance estimate
  achieves the lowest regret and best scaling with problem size.}
  \label{fig:deep_sea_results}
\end{figure*}

~\\
\paragraph{DeepSea.} First proposed by \citet{osband_deep_2019}, this environment tests
the agent's ability to explore over multiple time steps in the presence of a deterrent.
It consists of an $L \times L$ grid-world MDP, where the agent starts at the top-left
cell and must reach the lower-right cell. The agent decides to move left or right, while
always descending to the row below. We consider the deterministic version of the
problem, so the agent always transitions according to the chosen action. Going left
yields no reward, while going right incurs an action cost (negative reward) of $0.01 /
L$. The bottom-right cell yields a reward of 1, so that the optimal policy is to always
go right. As the size of the environment increases, the agent must perform sustained
exploration in order to reach the sparse reward. Implementation and hyperparameter
details are included in \cref{app:experimental_details}. 

The experiment consists on running each method for 1000 episodes and five random seeds, recording
the total regret and ``learning time'', defined as the first episode where the rewarding state has
been found at least in 10\% of the episodes so far \citep{odonoghue_variational_2021}. For this
experiment, we found that using $u_{\min} = -0.05$ improves the performance of our method: since the
underlying MDP is acyclic, propagating negative uncertainty rewards is consistent with our theory.

\begin{figure}[t]
	\centering
  \includegraphics{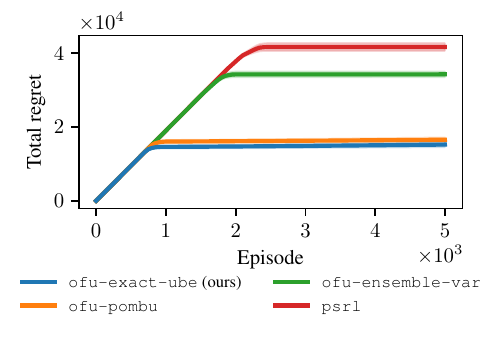}
  \caption{Total regret curve for the 7-room environment. Lower regret is better. Results are the average (solid lines) and
  standard error (shaded regions) over 10 random seeds. Our method achieves the lowest regret,
  significantly outperforming PSRL.}
  \label{fig:nroom}
\end{figure}

\cref{fig:deep_sea_results} (left) shows the evolution of learning time as $L$ increases. Our method
achieves the lowest learning time and best scaling with problem size. Notably, all the OFU-based
methods learn faster than PSRL, a strong argument in favour of using the variance of value functions
to guide exploration. \Cref{fig:deep_sea_results} (right) shows that our approach consistently
achieves the lowest total regret across all values of $L$. This empirical evidence indicates that
the solution to our UBE can be integrated into common exploration techniques like UCB to serve as an
effective uncertainty signal. Moreover, our method significantly improves peformance over
\texttt{pombu}, highlighting the relevance of our theory results.

~\\
\paragraph{7-room.} As implemented by \citet{domingues_rlberry_2021}, the 7-room environment
consists of seven connected rooms of size $5\times 5$. The agent starts in the center of the middle
room and an episode lasts 40 steps. The possible actions are up-down-left-right and the agent
transitions according to the selected action with probability $0.95$, otherwise it lands in a random
neighboring cell. The environment has zero reward everywhere except two small rewards at the start
position and in the left-most room, and one large reward in the right-most room. Unlike
\emph{DeepSea}, the underlying MDP for this environment contains cycles, so it evaluates our method
beyond the theoretical assumptions. In \cref{fig:nroom}, we show the regret curves over 5000
episodes. Our method achieves the lowest regret, which is remarkable considering recent empirical
evidence favoring PSRL over OFU-based methods in these type of environments
\citep{tiapkin_dirichlet_2022}. The large gap between \texttt{ensemble-var} and the UBE-based
methods is due to overall larger variance estimates from the former, which consequently requires
more episodes to reduce the value uncertainty.

\subsection{DeepMind Control Suite - Exploration Benchmark}
\label{subsec:online_experiments}
In this section, we evaluate the performance of QU-SAC for online exploration in environments with
continuous state-action spaces. Implementation details and hyperparameters are included in
\Cref{app:online_deep_rl}.

We test the exploration capabilities of QU-SAC on a subset of environments from the DeepMind Control
(DMC) suite \citep{tunyasuvunakool_dm_control_2020} with a sparse reward signal. Moreover, we modify
the environments' rewards to include a small negative term proportional to the squared norm of the
action vector, similar to \citet{curi_efficient_2020}. Such action costs are relevant for
energy-constrained systems where the agent must learn to maximize the primary objective while
minimizing the actuation effort. However, the added negative reward signal may inhibit exploration
and lead to premature convergence to sub-optimal policies. In this context, we want to compare the
exploration capabilities of QU-SAC with the different variance estimates.

In \cref{fig:dmc_benchmark} we plot the performance of all baselines in our exploration benchmark
after 500 episodes (or equivalently, 500K environment steps). In addition to individual learning
curves, we aggregate performance across all environments and report the median and inter-quartile
mean (IQM) \citep{agarwal_deep_2021}. The results highlight that QU-SAC with our proposed
\texttt{upper-bound} variance estimate offers the best overall performance. The pendulum swingup
environment is a prime example of a task where the proposed approach excels: greedily optimizing for
mean values, like MBPO and \texttt{ensemble-mean}, does not explore enough to observe the sparse
reward; \texttt{ensemble-var} improves performance upon the greedy approach, but does not work
consistently across random seeds unlike \texttt{upper-bound}. In this case, the stronger exploration
signal afforded by propagating uncertainty through the $U$-net is key to maintain exploration
despite low variability on the critic ensemble predictions.

\begin{figure*}[t]
	\centering
  \includegraphics[width=1.0\textwidth]{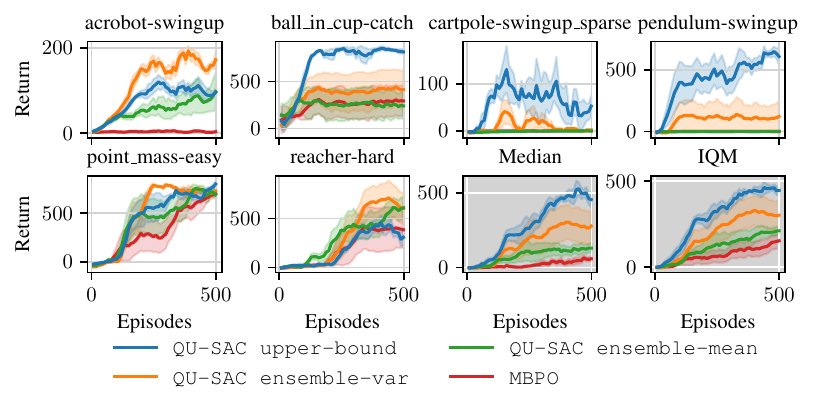}
  \caption{DeepMind Control Suite Benchmark smoothened learning curves over 500 episodes (500K
  environment steps). We report the mean (solid) and standard error (shaded region) over five random
  seeds. QU-SAC with the \texttt{upper-bound} variance estimate outperforms the baselines in 4/6
  environments and has the best overall performance.}
  \label{fig:dmc_benchmark}
\end{figure*}

\subsection{D4RL Offline Benchmark}
\label{subsec:offline_experiments}
In this section, we evaluate the performance of QU-SAC for offline RL in the Mujoco
\citep{todorov_mujoco_2012} datasets from the D4RL benchmark \citep{fu_d4rl_2020}. Implementation
details and hyperparameters are included in \Cref{app:offline_deep_rl}.

The core idea behind QU-SAC for offline optimization is to leverage the predicted value uncertainty
for conservative (pessimistic) policy optimization. This simply involves fixing $\lambda < 0$ to
downweight values depending on their predicted uncertainty. In addition to uncertainty-based
pessimism, prior work proposed SAC-$M$ \citep{an_uncertainty-based_2021} which uses an ensemble of
$M$ critics and imposes conservatism by taking the minimum of the ensemble prediction as the value
estimate. A key question we want to address with our experiments is whether pure uncertainty-based
pessimism is enough to avoid out-of-distribution over-estimation in offline RL. 

In order to provide an empirical answer, we augment QU-SAC with SAC-$M$ by training $M$ critics for
each of the $N$ dynamics models. The result is an ensemble of $NM$ critics, labelled as $Q_{ij}$ for
$i=\set{1, \dots, N}$, $j=\set{1, \dots, M}$. Each subset of $M$ critics is trained using clipped
Q-learning \citep{fujimoto_addressing_2018} as in SAC-$M$, where the $i$-th critic prediction is
simply defined as $Q_i(s,a) = \min_j Q_{ij}(s,a)$. The mean critic prediction is redefined as the
average over clipped $Q$-values, $\bar{Q}(s,a) = 1/N \sum_{i=1}^N \min_j Q_{ij}(s,a)$. If $M=1$ we
recover the original QU-SAC which only uses variance prediction as a mechanism for conservative
optimization. Note that MOPO fixes $M=2$, which means it combines uncertainty and clipped-based
conservatism by default; we re-implemented MOPO in order to allow for arbitrary $M$. In this
context, our key question becomes: can any of the methods perform well with $M=1$, i.e., only using
uncertainty-based pessimism?

\begin{figure*}[t]
	\centering
  \includegraphics[width=1.0\textwidth]{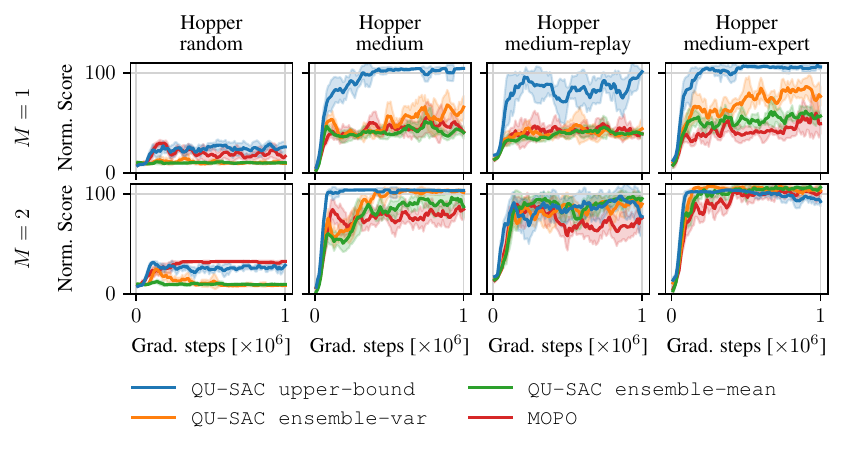}
  \caption{D4RL learning curves for Hopper datasets, smoothened by a moving average filter. We
  report the mean (solid) and standard deviation (shaded region) over five random seeds of the
  average normalized score over 10 evaluation episodes. We use $M=1$ for all baselines on the top
  row plots and $M=2$ on the bottom row. QU-SAC with the \texttt{upper-bound} variance estimate
  provides the most consistent performance across both values of $M$.}
  \label{fig:d4rl_hopper_critic_comparison}
\end{figure*}

\begin{table*}[t]
\addtolength{\tabcolsep}{-2pt}
\renewcommand\arraystretch{1.3}
\small
\centering
\caption{D4RL final normalized scores for model-based RL algorithms. The highest average scores are
highlighted in light blue. MOPO$\star$ corresponds to our own implementation of the algorithm by
\citet{yu_mopo_2020}, while QU-SAC utilizes the \texttt{upper-bound} variance estimate and $M=2$.
For MOPO$\star$ and QU-SAC, we report the mean and standard deviation over five random seeds. The
scores for the original MOPO results are as reported by \citet{bai_pessimistic_2022}. We take the
results for COMBO \citep{yu_combo_2021}, RAMBO \citep{rigter_rambo-rl_2022} and CBOP
\citep{jeong_conservative_2023} from their corresponding papers.}
\label{tab:d4rl_sota_comp}
% Somehow table caption overlaps with toprule if no vspace is used.
\resizebox{1.0\textwidth}{!}{%
\vspace{1ex}
\begin{tabular}{
  c@{\hspace{3pt}}l@{\hspace{5.0pt}}r@{\hspace{-0pt}}lr@{\hspace{-0.0pt}}lr@{\hspace{-0.0pt}}lr@{\hspace{-0.0pt}}lr@{\hspace{-0.0pt}}lr@{\hspace{-0.0pt}}lr
  }

\toprule

\multicolumn{1}{l}{} & & \multicolumn{2}{c}{MOPO} & \multicolumn{2}{c}{MOPO$\star$} &
\multicolumn{2}{c}{COMBO} & \multicolumn{2}{c}{RAMBO} & \multicolumn{2}{c}{CBOP} &
\multicolumn{2}{c}{\textbf{QU-SAC}}  \\

\midrule
\multirow{3}{*}{\rotatebox[origin=c]{90}{Random}} 

& HalfCheetah & 35.9 & $\pm$2.9  & 25.9 & $\pm$1.4 & 38.8 & $\pm$3.7& \colorbox{mine}{40.0} &
$\pm$2.3 & 32.8 & $\pm$0.4 & 30.2 & $\pm$1.5 \\

& Hopper & 16.7 & $\pm$12.2  & \colorbox{mine}{32.6} & $\pm$0.2  & 17.9 & $\pm$1.4 & 21.6 & $\pm$8.0
& 31.4 & $\pm$0.0 & 31.5 & $\pm$0.2 \\

& Walker2D & 4.2 & $\pm$5.7  & 1.0 & $\pm$1.9  & 7.0 & $\pm$3.6 & 11.5  & $\pm$10.5 & 17.8 &
$\pm$0.4 & \colorbox{mine}{21.7} & $\pm$0.1 \\

\midrule
\multirow{3}{*}{\rotatebox[origin=c]{90}{Medium}} 

& HalfCheetah & 73.1 & $\pm$2.4  & 60.6 & $\pm$2.4 & 54.2 & $\pm$1.5 & \colorbox{mine}{77.6} &
$\pm$1.5 & 74.3 & $\pm$0.2 & 60.7 & $\pm$1.7 \\

& Hopper & 38.3 & $\pm$34.9  & 81.3 & $\pm$15.7  & 97.2 & $\pm$2.2  & 92.8 & $\pm$6.0 &
102.6 & $\pm$0.1 & \colorbox{mine}{103.5} & $\pm$0.2\\

& Walker2D & 41.2 & $\pm$30.8  & 85.3 & $\pm$1.3  & 81.9 & $\pm$2.8  & 85.0 & $\pm$15.0 &
\colorbox{mine}{95.5} & $\pm$0.4 & 86.5 & $\pm$0.7 \\

\midrule
\multirow{3}{*}{\rotatebox[origin=c]{90}{\shortstack{Medium\\Replay}}} 

& HalfCheetah & \colorbox{mine}{69.2} & $\pm$1.1  & 55.7 & $\pm$0.9 & 55.1 & $\pm$1.0 & 68.9 &
$\pm$2.3 & 66.4 & $\pm$0.3 & 58.9 & $\pm$1.3 \\

& Hopper & 32.7 & $\pm$9.4  & 69.0 & $\pm$27.0  & 89.5 & $\pm$2.8 & 96.6  & $\pm$7.0 &
\colorbox{mine}{104.3} & $\pm$0.4 & 86.2 & $\pm$20.9 \\

& Walker2D & 73.7 & $\pm$9.4  & 83.1 & $\pm$5.0 & 56.0 & $\pm$8.6 & 85.0 & $\pm$15.0 &
\colorbox{mine}{92.7} & $\pm$0.9 & 76.8 & $\pm$0.6 \\

\midrule
\multirow{3}{*}{\rotatebox[origin=c]{90}{\shortstack{Medium\\Expert}}} 

& HalfCheetah & 70.3 & $\pm$21.9  & 95.0 & $\pm$1.7 & 90.0 & $\pm$5.6 & 93.7 & $\pm$10.5 &
\colorbox{mine}{105.4} & $\pm$1.6 & 99.1 & $\pm$2.5 \\

& Hopper & 60.6 & $\pm$32.5  & 104.5 & $\pm$7.7 & \colorbox{mine}{111.1} & $\pm$2.9
& 83.3 & $\pm$9.1 & \colorbox{mine}{111.6} & $\pm$0.2 & 93.8 & $\pm$10.4 \\

& Walker2D & 77.4 & $\pm$27.9  & 107.7 & $\pm$0.8  & 103.3 & $\pm$5.6 & 68.3 & $\pm$20.6 &
\colorbox{mine}{117.2} & $\pm$0.5 & 93.7 & $\pm$25.6 \\

\midrule
\multicolumn{1}{l}{} & \textbf{Average} & \multicolumn{2}{c}{49.4} & \multicolumn{2}{c}{66.8} &
\multicolumn{2}{c}{66.8} & \multicolumn{2}{c}{68.7} &\multicolumn{2}{c}{\colorbox{mine}{79.3}}
&\multicolumn{2}{c}{70.2} \\

% \multicolumn{1}{l}{} & \textbf{Median} & \multicolumn{2}{c}{50.9} & \multicolumn{2}{c}{76.8} &
% \multicolumn{2}{c}{69.0} & \multicolumn{2}{c}{80.4} &\multicolumn{2}{c}{94.1}
% &\multicolumn{2}{c}{81.9} \\

\multicolumn{1}{l}{} & \textbf{IQM} & \multicolumn{2}{c}{52.6} & \multicolumn{2}{c}{72.5} &
\multicolumn{2}{c}{71.1} & \multicolumn{2}{c}{78.0} &\multicolumn{2}{c}{\colorbox{mine}{89.3}}
&\multicolumn{2}{c}{77.1} \\
\bottomrule
\end{tabular}
}
\end{table*}

We conduct experiments in D4RL for three environments (Hopper, HalfCheetah and Walker2D) and four
tasks each (random, medium, medium-replay and medium-expert) for a total of 12 datasets. For each
dataset, we pre-train an ensemble of dynamics models and then run offline optimization for 1M
gradient steps. In \cref{fig:d4rl_hopper_critic_comparison} we present the results for the Hopper
datasets using $M=\set{1, 2}$. In the pure uncertainty-based pessimism setting ($M=1$), QU-SAC with
the \texttt{upper-bound} variance estimate obtains the best performance by a wide margin.
Qualitatively, the effect of supplementing \texttt{upper-bound} with clipped Q-learning ($M=2$) is
more stable performance rather than a significant score improvement, unlike most other baselines
that do improve substantially. These results suggest that proper uncertainty quantification might be
sufficient for offline learning, without relying on additional mechanisms to combat
out-of-distribution biases such as clipped Q-learning. Learning curves and scores for all datasets
are provided in \cref{app:offline_deep_rl}.

In \cref{tab:d4rl_sota_comp}, we compare the final scores of QU-SAC (using \texttt{upper-bound} and
$M=2$) against recent model-based offline RL methods. While scores are typically lower than the
state-of-the-art method CBOP \citep{jeong_conservative_2023}, our general-purpose method outperforms
MOPO and is on-par with more recent and stronger model-based baselines like COMBO and
RAMBO\footnote{When $M=1$, QU-SAC using \texttt{upper-bound} obtains an average score of $70.4$ (IQM
of 74.4) (see \cref{tab:d4rl_appendix} in the supplementary material), which is also comparable to
the reported performance of COMBO and RAMBO.}.

\section{Conclusions}
In this paper, we derived an uncertainty Bellman equation whose fixed-point solution converges to
the variance of values given a posterior distribution over MDPs. Our theory brings new understanding
by characterizing the gap in previous UBE formulations that upper-bound the variance of values. We
showed that this gap is the consequence of an over-approximation of the uncertainty rewards being
propagated through the Bellman recursion, which ignore the inherent \emph{aleatoric} uncertainty
from acting in an MDP. Instead, our theory recovers exclusively the \emph{epistemic} uncertainty due
to limited environment data, thus serving as an effective exploration signal. The tighter variance
estimate showed improved regret in typical tabular exploration problems.

Beyond tabular RL, we identified challenges on applying the UBE theory for uncertainty
quantification and proposed a simple proxy uncertainty reward to overcome them. Based on this
approximation, we introduced the $Q$-Uncertainty Soft Actor-Critic (QU-SAC) algorithm that can be
used for both online and offline RL with minimal changes. For online RL, the proposed proxy
uncertainty reward was instrumental for exploration in sparse reward problems. In offline RL, we
demonstrate QU-SAC has solid performance without additional regularization mechanisms unlike other
uncertainty quantification methods.

%%%%%%%%%%%%%%%%%%%%%%%%%%%%%%%%%%%
%%%%%% APPENDICES %%%%%%
%%%%%%%%%%%%%%%%%%%%%%%%%%%%%%%%%%
% \clearpage
\begin{appendices}
\maketitle
\section{Theory Proofs}
\label{app:proofs}
\subsection{Proof of \texorpdfstring{\cref{thm:ube}}{Theorem \ref{thm:ube}}}
\label{app:proofshm_ube}
In this section, we provide the formal proof of \cref{thm:ube}. We begin by showing an expression
for the posterior variance of the value function without assumptions on the MDP. We define the joint
distribution $p^\pi(a, s' \mid s) = \pi(a\mid s)p(s' \mid s,a)$ for a generic transition function
$p$. To ease notation, since $\pi$ is fixed, we will simply denote the joint distribution as $p(a,s'
\mid s)$. 

\begin{lemma}
  \label{lemma:variance_decomp_no_assumptions}
  For any $s \in \mathcal{S}$ and any policy $\pi$, it holds that
  \begin{equation}
    \label{eq:variance_decomposition}
    \V_{p \sim \Phi} \bracket{V^{\pi, p}(s)} = \gamma^2\E_{p \sim \Phi} \bracket{
      \paren{
        \sum_{a,s'}p(a, s' \mid s) V^{\pi, p}(s')
      }^2
    } - \gamma^2
    \paren{
      \E_{p \sim \Phi} \bracket{
        \sum_{a,s'}p(a,s' \mid s) V^{\pi, p}(s')
      }
    }^2.
  \end{equation}
\end{lemma}
\begin{proof}
  Using the Bellman expectation equation
  \begin{equation}
    \label{eq:bellman_expectation}
    V^{\pi,p}(s) = \sum_{a} \pi(a \mid s)r(s,a) + \gamma \sum_{a, s'}p(a,s' \mid s)V^{\pi, p}(s'),
  \end{equation}
  we have
  \begin{align}
    \V_{p \sim \Phi} \bracket{V^{\pi, p}(s)} &= \V_{p \sim \Phi} \bracket{\sum_a \pi(a \mid s)r(s,a) + \gamma\sum_{a, s'} p(a,s' \mid s)V^{\pi, p}(s')} \\
    &= \V_{p \sim \Phi} \bracket{\gamma\sum_{a,s'}p(a, s' \mid s)V^{\pi, p}(s')},  \label{eq:var_no_assumption}
  \end{align}
  where \cref{eq:var_no_assumption} holds since $r(s,a)$ is deterministic. Using the identity $\V[Y]
  = \E[Y^2] - (\E[Y])^2$ on \cref{eq:var_no_assumption} concludes the proof. 
\end{proof}

The next result is the direct consequence of our set of assumptions.
\begin{lemma}
  \label{lemma:independence_from_assumptions}
  Under \cref{assumption:transitions,assumption:acyclic}, for any $s \in \mathcal{S}$, any policy
  $\pi$, $\Cov[p(s' \mid s,a), V^{\pi,p}(s')] = 0$.
\end{lemma}
\begin{proof}
  Let $\tau_{0:H}$ be a random trajectory of length $H < \abs{\mathcal{S}}$ steps with the random
  transition dynamics $p$. Under \cref{assumption:acyclic},  $\tau_{0:H}$ is a sequence of $H$
  random, but \emph{unique} states $\set{s_0, s_1, \dots, s_{H-1}}$, i.e., we have $s_i \neq s_j$
  for all $i \neq j$. Moreover, under \cref{assumption:transitions}, the conditioned trajectory
  probability $\Prob(\tau_{0:H} \mid p)$, which is itself a random variable through conditioning on
  $p$, is a product of independent random variables defined by
  \begin{align}
    \Prob(\tau_{0:H} \mid p) &= \prod_{h=0}^{H-1}\pi(a_h \mid a_h)p(s_{h+1} \mid s_h, a_h) \\
    &= p(s_1 \mid s_0, a_0) \pi(a_0 \mid a_0)\prod_{h=1}^{H-1}\pi(a_h \mid s_h)p(s_{h+1} \mid s_h,
    a_h). \\
    &= p(s_1 \mid s_0, a_0)\pi(a_0 \mid s_0) \Prob(\tau_{1:H} \mid p).
  \end{align}
  Note that each transition probability in $\Prob(\tau_{0:H} \mid p)$ is distinct by
  \cref{assumption:acyclic} and there is an implicit assumption that the policy $\pi$ is independent
  of $p$. Then, for arbitrary $s_0=s$, $a_0 = a$ and $s_1 = s'$, we have that $p(s' \mid s, a)$ is
  independent of $\Prob(T_{1:H}\mid p)$. Since $V^{\pi, p}(s_1 \mid s_1 = s')$ is a function of
  $\Prob(T_{1:H}\mid p)$, then it is also independent of $p(s' \mid s, a)$. Finally, since
  independence implies zero correlation, the lemma holds.
\end{proof}

Using the previous result yields the following lemma.
\begin{lemma}
  \label{lemma:uncorrelated_property}
  Under \cref{assumption:transitions,assumption:acyclic}, it holds that
  \begin{equation}
    \label{eq:uncorrelated_property}
    \sum_{a, s'} \E_{p \sim \Phi} \bracket{
      p(a, s' \mid s) V^{\pi, p}(s')
    } = 
    \sum_{a, s'} \bar{p}(a, s' \mid s) \E_{p \sim \Phi} \bracket{V^{\pi, p}(s')}.
  \end{equation}
\end{lemma}
\begin{proof}
  For any pair of random variables $X$ and $Y$ on the same probability space, by definition of
  covariance it holds that $\E[XY] = \Cov[X,Y] + \E[X]\E[Y]$. Using this identity with
  \cref{lemma:independence_from_assumptions} and the definition of posterior mean transition
  \cref{eq:mean_model} yields the result.
\end{proof}

Now we are ready to prove the main theorem.
\ube*

\begin{proof}
  Starting from the result in \cref{lemma:variance_decomp_no_assumptions}, we consider each term on
  the r.h.s of \cref{eq:variance_decomposition} separately. For the first term, notice that within
  the expectation we have a squared expectation over the transition probability $p(s' \mid s,a)$,
  thus using the identity $(\E[Y])^2 = \E[Y^2] - \V[Y]$ results in
  \begin{align}
    \E_{p \sim \Phi} \bracket{
      \paren{
        \sum_{a,s'}p(a,s' \mid s) V^{\pi, p}(s')
      }^2
    } &= 
    \E_{p \sim \Phi} \bracket{
      \sum_{a,s'}p(a,s' \mid s)\paren{V^{\pi,p}(s')}^2 -
      \V_{a,s' \sim \pi, p} \bracket{V^{\pi,p}(s')}
    }.  \intertext{Applying linearity of expectation to bring it inside the sum
    and an application of \cref{lemma:uncorrelated_property} (note that the lemma applies for squared values as well) gives}
    &= 
    \label{eq:thm1_firsterm}
    \sum_{a,s'}\bar{p}(a,s' \mid s)\E_{p \sim \Phi}\bracket{\paren{V^{\pi,p}(s')}^2} - \E_{p \sim \Phi}\bracket{\V_{a,s' \sim \pi,p} \bracket{V^{\pi,p}(s')}
    }.
  \end{align}
  For the second term of the r.h.s of \cref{eq:variance_decomposition} we apply again
  \cref{lemma:uncorrelated_property} and under definition of variance
  \begin{align}
    \paren{
      \E_{p \sim \Phi} \bracket{
        \sum_{a,s'}p(a,s' \mid s) V^{\pi, p}(s')
      }
    }^2 \label{eq:prevhm1_seconderm}
    &= \paren{
      \sum_{a,s'}\bar{p}(a,s' \mid s)\E_{p \sim \Phi}\bracket{V^{\pi,p}(s')}
    }^2 \\
    &= \label{eq:thm1_seconderm} \sum_{a,s'}\bar{p}(a,s' \mid s)\paren{
      \E_{p \sim \Phi}\bracket{V^{\pi,p}(s')}
    }^2 - \V_{a,s' \sim\pi, \bar{p}}\bracket{\E_{p \sim \Phi}\bracket{V^{\pi,p}(s')}}.
  \end{align}
  Finally, since
  \begin{equation}
    \E_{p \sim \Phi}\bracket{\paren{V^{\pi,p}(s')}^2} -
    \paren{
      \E_{p \sim \Phi}\bracket{V^{\pi,p}(s')}
    }^2  = \V_{p \sim \Phi}\bracket{V^{\pi,p}(s')}
  \end{equation}
  for any $s' \in \mathcal{S}$, we can plug \cref{eq:thm1_firsterm,eq:thm1_seconderm} into
  \cref{eq:variance_decomposition}, which proves the theorem.
\end{proof}

\subsection{Proof of \texorpdfstring{\cref{thm:connection_uncertainties}}{Theorem
\ref{thm:connection_uncertainties}}}
\label{app:proofshm_connections}
In this section, we provide the supporting theory and the proof of
\cref{thm:connection_uncertainties}. First, we will use the identity $\V[\E[Y|X]] = \E[(\E[Y|X])^2]
- (\E[E[Y|X]])^2$ to prove $u(s) = w(s) - g(s)$ holds, with ${Y = \sum_{a,s'}p(a,s' \mid
s)V^{\pi,p}(s')}$. For the conditioning variable $X$, we define a transition function with fixed
input state $s$ as a mapping $p_s : \mathcal{A} \to \Delta(S)$ representing a distribution $p_s(s'
\mid a) = p(s' \mid s,a)$. Then $X = \mathbf{P}_s := \set{p_s(s' \mid a)}_{s' \in \mathcal{S}, a \in
\mathcal{A}}$. The transition function $p_s$ is drawn from a distribution $\Phi_{s}$ obtained by
marginalizing $\Phi$ on all transitions not starting from $s$. 
\begin{lemma}
  \label{lemma:pombu_local_uncertainty}
  Under \cref{assumption:transitions,assumption:acyclic}, it holds that 
  \begin{equation}
    \V_{p_s \sim \Phi_{s,t}} \bracket{\E_{p \sim \Phi} \bracket{\sum_{a,s'} p(a,s' \mid s)V^{\pi,p}(s') \given \mathbf{P}_s}} = \V_{p \sim \Phi} \bracket{\sum_{a,s'} p(a,s' \mid s)\bar{V}^\pi(s')}.
  \end{equation}
\end{lemma}
\begin{proof}
  Treating the inner expectation,
  \begin{align}
    \E_{p \sim \Phi} \bracket{\sum_{a,s'} p(a,s' \mid s)V^{\pi,p}(s')\mid \textbf{P}_s} &= \sum_{a}\pi(a \mid s) \sum_{s'}\E_{p \sim \Phi}\bracket{p(s' \mid s,a)V^{\pi, p}(s') \given \mathbf{P}_s}. 
    \intertext{Due to the conditioning, $p(s' \mid s,a)$ is deterministic within the expectation}
    &= \sum_{a,s'} p(a,s' \mid s) \E_{p \sim \Phi}\bracket{V^{\pi,p}(s') \given \mathbf{P}_s}.
    \intertext{By \cref{lemma:independence_from_assumptions}, $V^{\pi,p}(s')$ is independent of $\mathbf{P}_s$, so we can drop the conditioning}
    &= \sum_{a,s'} p(a,s' \mid s)\bar{V}^\pi(s').
  \end{align}
  Lastly, since drawing samples from a marginal distribution is equivalent to drawing samples from
  the joint, i.e., $\V_x[f(x)] = \V_{(x,y)}[f(x)]$, then:
  \begin{equation}
    \V_{p_s \sim \Phi_{s,t}}\bracket{\sum_{a,s'} p(a,s' \mid s)\bar{V}^\pi(s')} = \V_{p \sim \Phi}\bracket{\sum_{a,s'} p(a,s' \mid s)\bar{V}^\pi(s')},
  \end{equation}
  completing the proof.
\end{proof}

The next lemma establishes the result for the expression $\E[(\E[Y|X])^2]$.
\begin{lemma}
  \label{lemma:rhs_pombu_uncertainty_first}
  Under \cref{assumption:transitions,assumption:acyclic}, it holds that 
  \begin{align}
    \E_{p_s \sim \Phi_{s,t}} \bracket{\paren{\E_{p \sim \Phi} \bracket{\sum_{a,s'}p(a,s' \mid s)V^{\pi,p}(s')\given \mathbf{P}_s}}^2} &= \sum_{a,s'}\bar{p}(a,s' \mid s)\paren{\bar{V}^\pi(s')} - \E_{p \sim \Phi}\bracket{\V_{a,s' \sim \pi,p} \bracket{\bar{V}^\pi(s')}}.
  \end{align}
  \begin{proof}
    The inner expectation is equal to the one in \cref{lemma:pombu_local_uncertainty}, so we have
    that 
    \begin{align}
      \paren{\E_{p \sim \Phi} \bracket{\sum_{a,s'}p(a,s' \mid s)V^{\pi,p}(s')\given \mathbf{P}_s}}^2 &= \paren{\sum_{a,s'}p(a,s' \mid s) \bar{V}^\pi(s')}^2 \\
      &= \sum_{a,s'}p(a,s' \mid s)\paren{\bar{V}^\pi(s')}^2 - \V_{a,s' \sim \pi,p}\bracket{\bar{V}^\pi(s')}. \label{eq:term_no_exp}
    \end{align}
    Finally, applying expectation on both sides of \cref{eq:term_no_exp} yields the result.
  \end{proof}
\end{lemma}
Similarly, the next lemma establishes the result for the expression $(\E[\E[Y|X]])^2$.
\begin{lemma}
  \label{lemma:rhs_pombu_uncertainty_second}
  Under \cref{assumption:transitions,assumption:acyclic}, it holds that 
  \begin{align}
    \paren{\E_{p_s \sim \Phi_{s,t}} \bracket{\E_{p \sim \Phi} \bracket{\sum_{a,s'}p(a,s' \mid s)V^{\pi,p}(s')\given \mathbf{P}_s}}^2} &= \sum_{a,s'}\bar{p}(a,s' \mid s)\paren{\bar{V}^\pi(s')} - \V_{a,s' \sim \pi,\bar{p}} \bracket{\bar{V}^\pi(s')}.
  \end{align}
  \begin{proof}
    By the tower property of expectations, $(\E[\E[Y|X]])^2 = (\E[Y])^2$. Then, the result follows
    directly from \cref{eq:prevhm1_seconderm,eq:thm1_seconderm}.
  \end{proof}
\end{lemma}
The second part of \cref{thm:connection_uncertainties} is a corollary of the next lemma.
\begin{lemma}
  \label{lemma:inflated_uncertainty}
  Under \cref{assumption:transitions,assumption:acyclic}, it holds that
  \begin{equation}
    \label{eq:non_negative_quantity}
    \E_{p \sim \Phi}\bracket{\V_{a, s' \sim \pi, p}\bracket{V^{\pi,p}(s')} - \V_{a,s' \sim \pi, p}\bracket{\bar{V}^\pi(s')}}
  \end{equation}
  is non-negative.
\end{lemma}
\begin{proof}
  We will prove the lemma by showing \cref{eq:non_negative_quantity} is equal to $\E_{p \sim
  \Phi}\bracket{\V_{a,s' \sim \pi,p}\bracket{V^{\pi,p}(s') - \bar{V}^\pi(s')}}$, which is a
  non-negative quantiy by definition of variance. The idea is to derive two expressions for
  $\E[\V[Y|X]]$ and compare them. First, we will use the identity $\E[\V[Y|X]] = \E[\E[(Y - \E[Y|X])^2
  | X]]$. The outer expectation is w.r.t the marginal distribution $\Phi_{s}$ while the inner
  expectations are w.r.t $\Phi$. For the inner expectation we have
  \begin{align}
    &
    \E_{p \sim \Phi} \bracket{\paren{
      \sum_{a,s'}p(a,s' \mid s)V^{\pi, p}(s') - \E_{p \sim \Phi}\bracket{\sum_{a,s'}p(a,s' \mid s)V^{\pi,p}(s') \given \mathbf{P}_s}
    }^2\given \mathbf{P}_s} \\
    &= \E_{p \sim \Phi} \bracket{\paren{
      \sum_{a,s'}p(a,s' \mid s) \paren{
        V^{\pi, p}(s') - \E_{p\sim \Phi}\bracket{V^{\pi,p}\given \mathbf{P}_s}
      }}^2 \given \mathbf{P}_s} \\
    &= \E_{p \sim \Phi} \bracket{\paren{
      \sum_{a,s'}p(a,s' \mid s) \paren{
        V^{\pi, p}(s') - \bar{V}^\pi(s')
      }}^2 \given \mathbf{P}_s} \\
    &= \E_{p \sim \Phi} \bracket{
      \sum_{a,s'}p(a,s' \mid s)\paren{V^{\pi,p}(s') - \bar{V}^\pi(s')}^2 - \V_{a,s' \sim \pi,p}\bracket{V^{\pi,p}(s') - \bar{V}^\pi(s')} \given \mathbf{P}_s} \\
    &= \sum_{a,s'}p(a,s' \mid s) \V_{p \sim \Phi}\bracket{V^{\pi, p}(s')} - \E_{p \sim \Phi}\bracket{\V_{a,s' \sim \pi,p}\bracket{V^{\pi,p}(s') - \bar{V}^\pi(s')} \given \mathbf{P}_s}.
  \end{align}
  Applying the outer expectation to the last equation, along with
  \cref{lemma:independence_from_assumptions} and the tower property of expectations yields:
  \begin{equation}
    \label{eq:first_equivalenceotal_variance}
    \E[\V[Y|X]] = \sum_{a,s'}\bar{p}(a,s' \mid s)\V_{p \sim \Phi}\bracket{V^{\pi, p}(s')} - \E_{p \sim \Phi}\bracket{\V_{a,s' \sim \pi,p}\bracket{V^{\pi,p}(s') - \bar{V}^\pi(s')}}.
  \end{equation}
  Now we repeat the derivation but using $\E[\V[Y|X]] = \E[\E[Y^2|X] - (\E[Y|X])^2]$. For the inner
  expectation of the first term we have:
  \begin{align}
    &\E_{p \sim \Phi} \bracket{
      \paren{
        \sum_{a,s'}p(a,s' \mid s)V^{\pi,p}(s')
      }^2 \given \mathbf{P}_s
    }\\
    &= 
    \E_{p \sim \Phi} \bracket{
      \sum_{a,s'}p(a,s' \mid s)\paren{V^{\pi,p}(s')}^2 - \V_{a,s' \sim \pi,p}\bracket{V^{\pi,p}(s')} \given \mathbf{P}_s
    }.
  \end{align}
  Applying the outer expectation:
  \begin{equation}
    \label{eq:second_equivalenceotal_variance}
    \E[\E[Y^2|X]] = \sum_{a,s'}\bar{p}(a,s' \mid s)\E_{p \sim
    \Phi}\bracket{\paren{V^{\pi,p}(s')}^2} - \E_{p \sim
    \Phi}\bracket{\V_{a,s' \sim \pi,p}\bracket{V^{\pi,p}(s')}}.
  \end{equation}
  Lastly, for the inner expectation of $\E[(\E[Y|X])^2]$:
  \begin{align}
    \paren{\E_{p \sim \Phi} \bracket{
        \sum_{a,s'}p(a,s' \mid s)V^{\pi,p}(s')
      \given \mathbf{P}_s}
    }^2 &= 
    \paren{
      \sum_{a,s'}p(a,s' \mid s) \bar{V}^\pi(s')
    }^2 \\
    &=
    \sum_{a,s'}p(a,s' \mid s)\paren{\bar{V}^\pi(s')}^2 - \V_{a,s' \sim \pi,p}\bracket{\bar{V}^\pi(s')}.
  \end{align}
  Applying the outer expectation:
  \begin{equation}
    \label{eq:third_equivalenceotal_variance}
    \E[(\E[Y|X])^2] = \sum_{a,s'}\bar{p}(a,s' \mid s)\paren{\bar{V}^\pi(s')}^2 - \E_{p \sim \Phi}\bracket{\V_{a,s' \sim \pi,p}\bracket{\bar{V}^{\pi}(s')}}.
  \end{equation}
  Finally, by properties of variance, \cref{eq:first_equivalenceotal_variance} =
  \cref{eq:second_equivalenceotal_variance} - \cref{eq:third_equivalenceotal_variance} which
  gives the desired result.
\end{proof}
\connections*
\begin{proof}
  By definition of $u(s)$ in \cref{eq:bellman_exact_reward}, proving the claim is equivalent to
  showing
  \begin{equation}
    \label{eq:thm2_proxy_equivalence}
    \V_{a,s' \sim \pi, \bar{p}}\bracket{\bar{V}^\pi(s')} = w(s) + \E_{p \sim \Phi}\bracket{\V_{a,s' \sim \pi, p}\bracket{\bar{V}^\pi(s')}},
  \end{equation}
  which holds by combining
  \cref{lemma:pombu_local_uncertainty,lemma:rhs_pombu_uncertainty_first,lemma:rhs_pombu_uncertainty_second}.
  Lastly, $u(s) \leq w(s)$ holds by \cref{lemma:inflated_uncertainty}.
\end{proof}

\section{Theory Extensions}
\subsection{Unknown Reward Function}
\label{app:unknown_rewards}
We can easily extend the derivations on \cref{app:proofshm_ube} to include the additional
uncertainty coming from an \emph{unknown} reward function. Similarly, we assume the reward function
is a random variable $r$ drawn from a prior distribution $\Psi(r)$, and whose belief will be updated
via Bayes rule. In this new setting, we now consider the variance of the values under the
distribution of MDPs, represented by the random variable $\mathcal{M}$. We need the following
additional assumptions to extend our theory.
\begin{assumption}[Independent rewards]
    \label{assumption:indep_rewards}
    $r(x,a)$ and $r(y,a)$ are independent random variables if $x\neq y$.
\end{assumption}
\begin{assumption}[Independent transitions and rewards]
    \label{assumption:indepransit_rewards}
    The random variables $p(\cdot \mid s,a)$ and $r(s,a)$ are independent for any $(s,a)$.
\end{assumption}
With \cref{assumption:indep_rewards} we have that the value function of next states is independent
of the transition function and reward function at the current state.
\cref{assumption:indepransit_rewards} means that sampling $\mathcal{M} \sim \Gamma$
is equivalent as independently sampling $p \sim \Phi$ and $r \sim \Psi$.

\begin{restatable}{theorem}{ube_unknown_reward}
  \label{thm:ube_unknown_rewards}
  Under \crefrange{assumption:transitions}{assumption:indepransit_rewards}, for any
  $s \in \mathcal{S}$ and policy $\pi$, the posterior variance of the value function, $U^\pi =
  \V_{\mathcal{M} \sim \Gamma} \bracket{V^{\pi,\mathcal{M}}}$ obeys the uncertainty Bellman equation
  \begin{equation}
    U^\pi(s) = \V_{r \sim \Psi} \bracket{
      \sum_a \pi(a \mid s) r(s,a)
    } + 
    \gamma ^ 2u(s) + \gamma^2\sum_{a, s'}\pi(a \mid s)\bar{p}(s' \mid s,a) U^\pi(s'),
  \end{equation}
  where $u(s)$ is defined in \cref{eq:bellman_exact_reward}.
\end{restatable}
\begin{proof}
  By \cref{assumption:indep_rewards,assumption:indepransit_rewards} and following the derivation
  of \cref{lemma:variance_decomp_no_assumptions} we have
  \begin{align}
    \V_{\mathcal{M} \sim \Gamma} \bracket{V^{\pi, \mathcal{M}}(s)} &= \V_{\mathcal{M} \sim \Gamma} \bracket{\sum_a \pi(a \mid s)r(s,a) + \gamma\sum_{a, s'} p(a,s' \mid s)V^{\pi, \mathcal{M}}(s')} \\
    &= \V_{r \sim \Psi} \bracket{
      \sum_a \pi(a \mid s) r(s,a)
    } + \V_{\mathcal{M} \sim \Gamma} \bracket{\gamma\sum_{a,s'}p(a, s' \mid s)V^{\pi, \mathcal{M}}(s')}.
  \end{align}
  Then following the same derivations as \cref{app:proofshm_ube} completes the proof.
\end{proof}

\subsection{Extension to \texorpdfstring{$Q$}{Q}-values}
\label{app:extension_q_values}
Our theoretical results naturally extend to action-value functions. The following result is
analogous to \cref{thm:ube}.
\begin{restatable}{theorem}{ube_q}
  Under \cref{assumption:transitions,assumption:acyclic}, for any $(s, a) \in \mathcal{S} \times
  \mathcal{A}$ and policy $\pi$, the posterior variance of the $Q$-function, $U^\pi = \V_{p \sim
  \Phi} \bracket{Q^{\pi,p}}$ obeys the uncertainty Bellman equation
  \begin{equation}
    U^\pi(s,a) = 
    \gamma ^ 2u(s,a) + \gamma^2\sum_{a', s'}\pi(a' \mid s')\bar{p}(s' \mid s,a) U^\pi(s', a'),
  \end{equation}
  where $u(s,a)$ is the local uncertainty defined as
  \begin{equation}
    u(s,a) = \V_{a', s' \sim \pi, \bar{p}} \bracket{\bar{Q}^\pi(s', a')} -
    \E_{p \sim \Phi} \bracket{\V_{a', s' \sim \pi, p} \bracket{Q^{\pi, p}(s',a')}}
  \end{equation}
\end{restatable}
\begin{proof}
  Follows the same derivation as \cref{app:proofshm_ube}
\end{proof}
Similarly, we can connect to the upper-bound found by \citet{zhou_deep_2020} with the following
theorem.
\begin{restatable}{theorem}{connections_q_function}
  Under \cref{assumption:transitions,assumption:acyclic}, for any $(s, a) \in \mathcal{S} \times
  \mathcal{A}$ and policy $\pi$, it holds that $u(s,a) = w(s,a) - g(s,a)$, where $w(s,a) =
  \V_{p\sim \Phi} \bracket{\sum_{a', s'}\pi(a' \mid s') p(s' \mid s,a) \bar{Q}^\pi(s',a')}$ and
  $g(s,a) = \E_{p \sim \Phi}\bracket{\V_{a',s' \sim \pi, p} \bracket{Q^{\pi,p}(s',a')}-
  \V_{a',s' \sim \pi, p} \bracket{\bar{Q}^\pi(s',a')}}$. Furthermore, we have that the gap $g(s, a)
  \geq 0$ is non-negative, thus $u(s,a) \leq w(s,a)$.
\end{restatable}
\begin{proof}
  Follows the same derivation as \cref{app:proofshm_connections}. Similarly, we can prove that the
  gap $g(s,a)$ is non-negative by showing it is equal to $\E_{p \sim
  \Phi}\bracket{\V_{a',s' \sim \pi, p} \bracket{Q^{\pi,p}(s',a') - \bar{Q}^\pi(s',a')}}$.
\end{proof}

\subsection{State-Action Uncertainty Rewards}
\label{app:q_uncertainty_rewards}
In our practical experiments, we use the results of both
\Cref{app:unknown_rewards,app:extension_q_values} to compose the uncertainty rewards propagated via
the UBE. Concretely, we consider the following two approaches for computing state-action uncertainty
rewards:
\begin{itemize}
  \item \texttt{pombu}: 
  \begin{equation}
    \label{eq:pombu_q_rewards}
    w(s,a) = \V_{p\sim \Phi} \bracket{\sum_{a', s'}\pi(a' \mid s') p(s' \mid s,a)
    \bar{Q}^\pi(s',a')}
  \end{equation}
  \item \texttt{exact-ube}:
  \begin{equation}
    \label{eq:exact_ube_q_rewards}
    u(s,a) = w(s,a) - \E_{p \sim
    \Phi}\bracket{\V_{a',s' \sim \pi, p} \bracket{Q^{\pi,p}(s',a') - \bar{Q}^\pi(s',a')}}
  \end{equation}
\end{itemize}

Additionally, since we also learn the reward function, we add to the above the uncertainty term
generated by the reward function posterior, as shown in \cref{app:unknown_rewards}:
$\V_{r \sim \Psi} \bracket{r(s,a)}$.

\section{Tabular Environments Experiments}
In this section, we provide more details about the tabular implementation of
\Cref{algorithm:our_algorithm} and environment details.

\subsection{Implementation Details}
\label{app:experimental_details}
\paragraph{Model learning.} For the transition function we use a prior
$\text{Dirichlet}(1/\sqrt{S})$ and for rewards a standard normal $\mathcal{N}(0,1)$, as done by
\citet{odonoghue_making_2019}. The choice of priors leads to closed-form posterior updates based on
state-visitation counts and accumulated rewards. We add a terminal state to our modeled MDP in order
to compute the values in closed-form via linear algebra. 

~\\
\paragraph{Accelerating learning.} For the \emph{DeepSea} benchmark we accelerate learning by
imagining each experienced transition $(s, a, s', r)$ is repeated $L$ times, as initially suggested
in \citet{osband_deep_2019} (see footnote $9$), although we scale the number of repeats with the
size of the MDP. Effectively, this strategy forces the MDP posterior to shrink faster, thus making
all algorithms converge in fewer episodes. The same strategy was used for all the methods evaluated
in the benchmark.

~\\
\paragraph{Policy optimization.} All tested algorithms (PSRL and OFU variants) optimize the policy
via policy iteration, where we break ties at random when computing the $\argmax$, and limit the
number of policy iteration steps to $40$.

~\\
\paragraph{Hyperparameters.} Unless noted otherwise, all tabular RL experiments use a discount
factor $\gamma=0.99$, an exploration gain $\lambda = 1.0$ and an ensemble size $N=5$.

~\\
\paragraph{Uncertainty reward clipping.} For \emph{DeepSea} we clip uncertainty rewards with
$u_{\min} = -0.05$ and for the 7-room environment we keep $u_{\min} = 0.0$.

\subsection{Environment Details}
\paragraph{\emph{DeepSea}.} As proposed by \citet{osband_deep_2019}, \emph{DeepSea} is a grid-world
environment of size $L \times L$, with $S = L^2$ and $A = 2$.

~\\
\paragraph{7-room.} As implemented by \citet{domingues_rlberry_2021}, the 7-room environment
consists of seven connected rooms of size $5\times 5$, represented as an MDP of size $S=181$ and
discrete action space with size $A=4$. The starting state is always the center cell of the middle
room, which yields a reward of $0.01$. The center cell of the left-most room gives a reward of $0.1$
and the center cell of the right-most room gives a large reward of $1$. The episode terminates after
$40$ steps and the state with large reward is absorbing (i.e., once it reaches the rewarding state,
the agent remains there until the end of the episode). The agent transitions according to the
selected action with probability $0.95$ and moves to a randomly selected neighboring cell with
probability $0.05$.

\section{Online Deep RL Experiments}
\label{app:online_deep_rl}
In this section, we provide details regarding the online implementation of QU-SAC. Also, we include
relevant hyperparameters, environment details and additional results.

\subsection{Implementation Details}
\label{app:deep_rl_implementation}
We build QU-SAC on top of MBPO \citep{janner_when_2019} following \Cref{algorithm:qusac_online}. The
main differences with the original implementation are as follows:
\begin{itemize}
  \item In \Cref{line:model_rollouts}, we perform a total of $N+1$ $k$-step rollouts corresponding
  to both the model-randomized and model-consistent rollout modalities. The original MBPO only
  executes the former to fill up $\mathcal{D}_{\text{model}}$.
  \item In \Cref{line:q_update}, we update the ensemble of $Q$-functions on the corresponding
  model-consistent buffer. MBPO trains twin critics (as in SAC) on mini-batches from
  $\mathcal{D}_{\text{model}}$.
  \item In \Cref{line:u_update}, we update the $U$-net for the UBE-based variance estimation
  methods.
  \item In \Cref{line:pi_update}, we update $\pi_\phi$ by maximizing the uncertainty-aware
  $Q$-values. MBPO maximizes the minimum of the twin critics (as in SAC). Both approaches include an
  entropy maximization term.
\end{itemize}
\begin{algorithm}[tb]
   \caption{QU-SAC (online)}
   \label{algorithm:qusac_online}
\begin{algorithmic}[1]
  \STATE Initialize policy $\pi_{\phi}$, predictive model $p_{\theta}$, critic ensemble
  $\set{Q_i}_{i=1}^{N}$, uncertainty net $U_\psi$ (optional), environment dataset $\mathcal{D}$,
  model datasets $\mathcal{D}_{\text{model}}$ and $\set{\mathcal{D}_i}_{i=1}^{N}$.

  \STATE global step $\leftarrow 0$
  \FOR{episode $t=0, \dots, T-1$}
    \FOR{$E$ steps}
      \IF{global step \% $F == 0$}
        \STATE Train model $p_{\theta}$ on $\mathcal{D}$ via maximum likelihood
        \FOR{$M$ model rollouts}

          \STATE Perform $k$-step model rollouts starting from $s \sim \mathcal{D}$; add to
          $\mathcal{D}_{\text{model}}$ and $\set{\mathcal{D}_i}_{i=1}^{N}$\label{line:model_rollouts}
        \ENDFOR
      \ENDIF

      \STATE Take action in environment according to $\pi_{\phi}$; add to $\mathcal{D}$
      \FOR{$G$ gradient updates}
        \STATE Update $\set{Q_i}_{i=1}^{N}$ with mini-batches from $\set{\mathcal{D}_i}_{i=1}^{N}$, via SGD on \cref{eq:loss_q} \label{line:q_update}
        \STATE (Optional) Update $U_\psi$ with mini-batches from $\mathcal{D}_{\text{model}}$, via SGD on \cref{eq:loss_u} \label{line:u_update}
        \STATE Update $\pi_\phi$ with mini-batches from $\mathcal{D}_{\text{model}}$, via SGD on \cref{eq:actor_loss} \label{line:pi_update}
      \ENDFOR
    \ENDFOR
    \STATE global step $\leftarrow$ global step $+ 1$
  \ENDFOR
\end{algorithmic}
\end{algorithm}

The main hyperparameters for our experiments are included in \Cref{tab:hparams_online}. Further
implementation details are now provided. 
\renewcommand{\arraystretch}{1.1}
\begin{table*}[t]
\caption{Hyperparameters for the DeepMind Control Suite experiments of \cref{subsec:online_experiments}. For MBPO, the only deviation from the listed parameters is the use of $M=2$ as the original method uses clipped Q-learning.}
\label{tab:hparams_online}
\begin{center}
\resizebox{1.0\textwidth}{!}{%
\begin{tabular}{|c|c|}
\toprule
\textbf{Name}  & \textbf{Value} \\
\midrule
\multicolumn{2}{|c|}{\textbf{General}} \\
\midrule
$T$ - \# episodes & $500$\\
$E$ - steps per episode & $10^3$ \\ 
Replay buffer $\mathcal{D}$ capacity & $10^5$ \\
Batch size (all nets) & $256$ \\
Warm-up steps (under initial policy) & $5 \times 10^3$ \\
\midrule
\multicolumn{2}{|c|}{\textbf{SAC}} \\
\midrule
$G$ - \# gradient steps & $10$\\
Auto-tuning of entropy coefficient $\alpha$? & Yes  \\
Target entropy & $-\text{dim}(\mathcal{A})$ \\
Actor MLP network & 2 hidden layers - 128 neurons - Tanh activations \\
Critic MLP network & 2 hidden layers - 256 neurons - Tanh activations \\
Actor/Critic learning rate & $3\times 10^{-4}$ \\
\midrule
\multicolumn{2}{|c|}{\textbf{Dynamics Model}} \\
\midrule
$N$ - ensemble size & $5$\\
$F$ - frequency of model training (\# steps) & $250$ \\
$L$ - \# model rollouts per step & $400$ \\
$k$ - rollout length & $5$ \\
$\Delta$ - \# Model updates to retain data & $1$ \\
Model buffer(s) capacity & $L \times F \times k \times \Delta = 5 \times 10^5$
\\
Model MLP network  & 4 layers - 200 neurons -
SiLU activations \\
Learning rate & $1\times 10^{-3}$\\
\midrule
\multicolumn{2}{|c|}{\textbf{QU-SAC Specific}} \\
\midrule
$M$ - \# critics per dynamics model & $1$\\
$\lambda$ - \# uncertainty gain & $1.0$ \\
Uncertainty type & \set{\texttt{ensemble-var}, \texttt{upper-bound}} \\
\bottomrule
\end{tabular}
}
\end{center}
% \vspace{1in}
\end{table*}

\noindent\paragraph{Model learning.} We leverage the \texttt{mbrl-lib} Python library from
\citet{pineda_mbrl-lib_2021} and train an ensemble of $N$ probabilistic neural networks. We use the
default MLP architecture with four layers of size 200 and SiLU activations. The networks predict
delta states, $\Delta = s' - s$, and receive as input state-action pairs. We use the default
initialization of the network provided by the library, which samples weights from a truncated
Gaussian distribution, however we found it helpful to increase by a factor of $2.0$ the standard
deviation of the truncated Gaussian; a wider distribution of weights allows for more diverse dynamic
models at the beginning of training.

\noindent\paragraph{Model-generated buffers.} The capacity of the model-generated buffers
$\mathcal{D}_{\text{model}}$ and $\set{\mathcal{D}^i_{\text{model}}}_{i=1}^{N}$ is computed as $k
\times M \times F \times \Delta$, where $\Delta$ is the number of
model updates before entirely overwriting the buffers. Larger values of this parameter allows for
more off-policy (old) data to be stored and sampled for training.

\noindent\paragraph{SAC specifics.} Our SAC implementation is based on the open-source repository
\url{https://github.com/pranz24/pytorch-soft-actor-critic}, as done by \texttt{mbrl-lib}. For all
our experiments, we use the automatic entropy tuning flag that adaptively modifies the entropy gain
$\alpha$ based on the stochasticity of the policy.

\subsection{Environment Details}
We take a subset of sparse reward environments from the DeepMind Control Suite and include an
additional action cost proportional to the squared norm of the action taken by the agent. Namely,
\begin{equation}
  \text{\texttt{action\textunderscore cost}} = \rho \sum_{i=1}^{\abs{\mathcal{A}}} a_i^2
\end{equation} 
where $\rho$ is an environment specific multiplier, $a_i$ is the $i$-th component of the action
vector and $\abs{\mathcal{A}}$ is the size of the action space. For \texttt{acrobot},
\texttt{reacher-hard} and \texttt{cartpole-swingup} we use $\rho=0.01$; for \texttt{pendulum} and
\texttt{point-mass} we use $\rho = 0.05$; and lastly, for \texttt{ball-in-cup} we use $\rho=0.2$.

\section{Offline Deep RL Experiments}
\label{app:offline_deep_rl}
In this section, we provide further details regarding the use of QU-SAC for offline optimization,
which includes a detailed algorithmic description, hyperparameters and learning curves not included
in the main body of the paper.

\begin{algorithm}[tb]
   \caption{QU-SAC (offline)}
   \label{algorithm:qusac_offline}
\begin{algorithmic}[1]
  \STATE Initialize policy $\pi_{\phi}$, predictive model $p_{\theta}$, critic ensemble
  $\set{Q_i}_{i=1}^{N}$, uncertainty net $U_\psi$ (optional), offline dataset $\mathcal{D}$,
  model datasets $\mathcal{D}_{\text{model}}$ and $\set{\mathcal{D}_i}_{i=1}^{N}$.

  \STATE Train model $p_{\theta}$ on $\mathcal{D}$ via maximum likelihood
  \FOR{steps $g=0, \dots, G-1$}

    \IF{g \% $F == 0$}
      \FOR{$L$ model rollouts}
        \STATE Perform $k$-step model rollouts starting from $s \sim \mathcal{D}$; add to $\mathcal{D}_{\text{model}}$ and $\set{\mathcal{D}_i}_{i=1}^{N}$
      \ENDFOR
    \ENDIF

    \STATE Update $\set{Q_i}_{i=1}^{N}$ with mini-batches from $\set{\mathcal{D} \cup \mathcal{D}_i}_{i=1}^{N}$, via SGD on \cref{eq:loss_q}
    \STATE (Optional) Update $U_\psi$ with mini-batches from $\mathcal{D} \cup \mathcal{D}_{\text{model}}$, via SGD on \cref{eq:loss_u}
    \STATE Update $\pi_\phi$ with mini-batches from $\mathcal{D} \cup \mathcal{D}_{\text{model}}$, via SGD on \cref{eq:actor_loss}
  \ENDFOR
\end{algorithmic}
\end{algorithm}

\subsection{Implementation Details}
\label{app:deep_rl_offline_implementation}
We modify the online version of QU-SAC described in \cref{algorithm:qusac_online} to reflect the
execution flow of offline optimization, which we present in \cref{algorithm:qusac_offline}. The
hyperparameters used for the reported results are included in \cref{tab:hparams_offline}. Beyond the
algorithmic changes, we now list the main implementation details differing from the online
implementation of QU-SAC:
\renewcommand{\arraystretch}{1.1}
\begin{table*}[t]
\caption{Hyperparameters for the D4RL experiments of \cref{subsec:offline_experiments}.}
\label{tab:hparams_offline}
\begin{center}
\resizebox{1.0\textwidth}{!}{%
\begin{tabular}{|c|c|}
\toprule
\textbf{Name}  & \textbf{Value} \\
\midrule
\multicolumn{2}{|c|}{\textbf{General}} \\
\midrule
$G$ - gradient steps & $10^6$ \\ 
Replay buffer $\mathcal{D}$ capacity & $10^6$ \\
Batch size (all nets) & $512$ \\
\midrule
\multicolumn{2}{|c|}{\textbf{SAC}} \\
\midrule
Auto-tuning of entropy coefficient $\alpha$? & Yes  \\
Target entropy & $-\text{dim}(\mathcal{A})$ \\
Actor MLP network & 3 hidden layers - 256 neurons - Tanh activations \\
Critic MLP network & 3 hidden layers - 256 neurons - Tanh activations \\
Actor learning rate & $3\times10^{-5}$ \\
Critic learning rate & $3\times 10^{-4}$ \\
\midrule
\multicolumn{2}{|c|}{\textbf{Dynamics Model}} \\
\midrule
$N$ - ensemble size & $5$\\
$F$ - frequency of data collection (\# steps) & $1000$ \\
$L$ - rollout batch size & $5\times 10^4$ \\
$k$ - rollout length & $15$ \\
$\Delta$ - \# Data collection calls to retain data & $5$ \\
Model buffer(s) capacity & $L \times k \times \Delta = 3.75 \times 10^6$
\\
Model MLP network  & 4 layers - 200 neurons -
SiLU activations \\
Learning rate & $1\times 10^{-3}$\\
\midrule
\multicolumn{2}{|c|}{\textbf{QU-SAC Specific}} \\
\midrule
$M$ - \# critics per dynamics model & $\set{1, 2}$\\
$\lambda$ - \# uncertainty gain & $-1.0$ \\
Uncertainty type & \set{\texttt{ensemble-var}, \texttt{upper-bound}} \\
\bottomrule
\end{tabular}
}
\end{center}
% \vspace{1in}
\end{table*}

\noindent\paragraph{Model learning.} The only difference w.r.t. the online setting is that the we
normalize the state-action inputs to the model, where the normalization statistics are calculated
based on the offline dataset $\mathcal{D}$.

\noindent\paragraph{Data mixing.} In Lines 7-9, we highlight that, in contrast to the online
setting, the mini-batches used to train the critic, actor and $U$-net mix both model-generated and
offline data. In particular, we use a fixed 50/50 split between these two data sources for all our
experiments (inclusing QU-SAC and MOPO).

\noindent\paragraph{MOPO details.} In order to conduct a fair comparison between MOPO and QU-SAC, we
implement MOPO in our codebase so that it shares the same core components as our QU-SAC
implementation. After initial testing of our MOPO implementation, we found that using an uncertainty
penalty of $\lambda=1.0$ worked well across datasets. Note that our implementation of MOPO (labeled
MOPO$\star$ in \cref{tab:d4rl_sota_comp}) significantly outperforms the scores reported by
\citet{bai_pessimistic_2022}, which were obtained by running the original codebase by
\citet{yu_mopo_2020} but on the v2 datasets from D4RL.

\subsection{Dataset Details}
We use the v2 version of D4RL datasets and evaluate using the normalized scores provided by the
software package.
\begin{figure*}[ht!]
	\centering
  \includegraphics[width=0.9\textwidth]{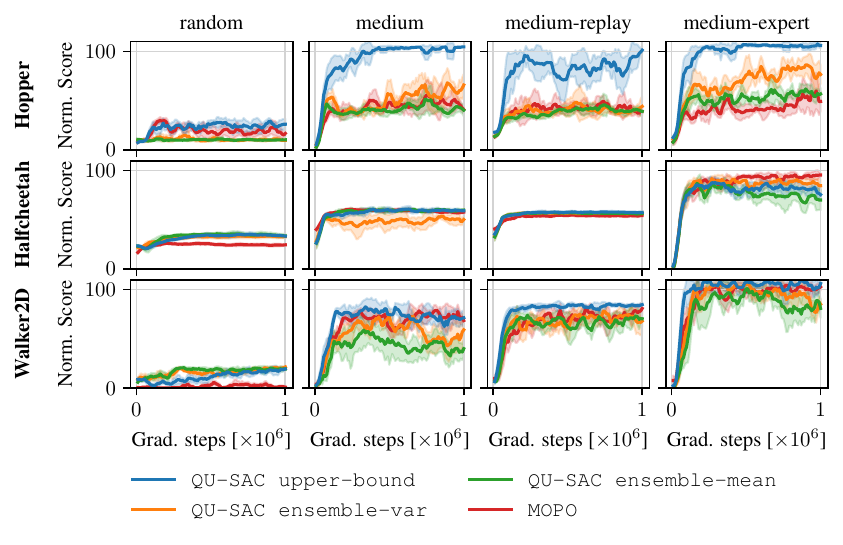}
  \caption{D4RL smoothened learning curves for $M=1$. We report the mean and standard deviation over five random seeds of the average normalized score over 10 evaluation episodes.}
  \label{fig:full_d4rl_m1}
\end{figure*}

\begin{figure*}[ht!]
	\centering
  \includegraphics[width=0.9\textwidth]{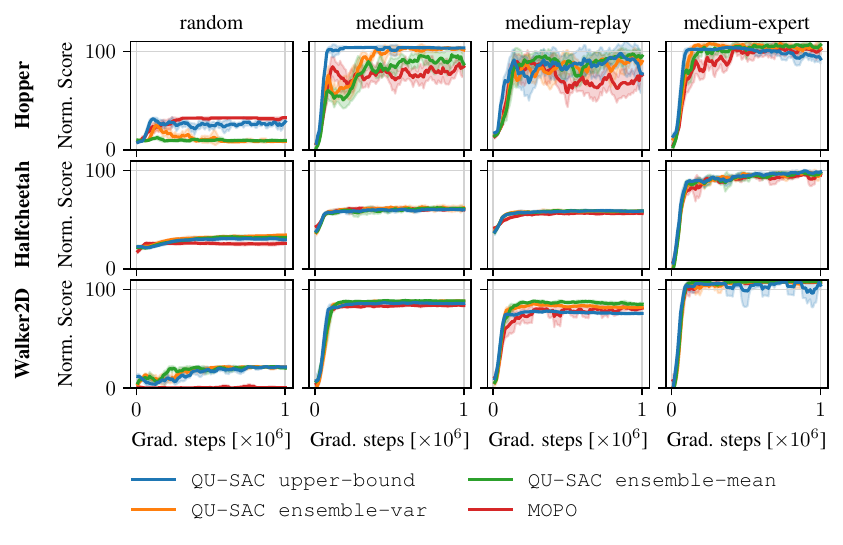}
  \caption{D4RL smoothened learning curves for $M=2$. We report the mean and standard deviation over five random seeds of the average normalized score over 10 evaluation episodes.}
  \label{fig:full_d4rl_m2}
\end{figure*}

\subsection{D4RL Learning Curves \& Scores}
In \cref{fig:full_d4rl_m1,fig:full_d4rl_m2} we include all the learning curves for $M = \set{1, 2}$,
respectively. For each run, we report the average normalized score over 10 evaluation episodes.
These averaged scores are then also averaged over five independent random seeds to obtain the
reported learning curves. In \cref{tab:d4rl_appendix} we report the associated final scores after 1M
gradient steps.

We observe that for $M=2$ \texttt{upper-bound} has lower overall performance than
\texttt{ensemble-var}. We believe this difference in performance is largely due to using a fixed
value of $\lambda=-1.0$ for all experiments. Since using $M=2$ alread acts a strong regularizer in
the offline setting, \texttt{upper-bound} would likely benefit from using a lower magnitude
$\lambda$ given the (empirically) larger uncertainty estimates compared to \texttt{ensemble-var}.

\begin{table*}[ht!]
\addtolength{\tabcolsep}{-2pt}
\renewcommand\arraystretch{1.5}
\centering
\caption{D4RL scores after 1M gradient steps. We report the mean and standard deviation over five random seeds of the average normalized score across 10 evaluation episodes. We highlight the highest mean scores for each value of $M$.}
\label{tab:d4rl_appendix}
\vspace{1ex}
\resizebox{1.0\textwidth}{!}{%
% Somehow table caption overlaps with toprule if no vspace is used.
\begin{tabular}{
  c@{\hspace{3pt}}l@{\hspace{5.0pt}}r@{\hspace{-0pt}}lr@{\hspace{-0.0pt}}lr@{\hspace{-0.0pt}}lr@{\hspace{-0.0pt}}l|r@{\hspace{-0.0pt}}lr@{\hspace{-0.0pt}}lr@{\hspace{-0.0pt}}lr@{\hspace{-0.0pt}}lr}

\toprule

\multicolumn{1}{l}{} & & \multicolumn{8}{c|}{$M=1$} &  \multicolumn{8}{c}{$M=2$} \\

\multicolumn{1}{l}{} & & \multicolumn{2}{c}{MOPO$\star$} & \multicolumn{2}{c}{e-mean} &
\multicolumn{2}{c}{e-var} & \multicolumn{2}{c|}{u-bound} & \multicolumn{2}{c}{MOPO$\star$} &
\multicolumn{2}{c}{e-mean} & \multicolumn{2}{c}{e-var} & \multicolumn{2}{c}{u-bound}  \\

\midrule
\multirow{3}{*}{\rotatebox[origin=c]{90}{Random}} 

& HalfCheetah & 24.8 & $\pm$0.7  & \colorbox{mine}{33.6} & $\pm$3.8 & 33.0 & $\pm$1.7& 33.4 &
$\pm$1.2 & 25.9 & $\pm$1.4 & 32.4 & $\pm$1.8 & \colorbox{mine}{34.8} & $\pm$1.0 & 30.2 & $\pm$1.5 \\

& Hopper & 20.7 & $\pm$9.2  & 10.3 & $\pm$0.8  & 9.3 & $\pm$1.8 & \colorbox{mine}{28.3} & $\pm$8.3 &
\colorbox{mine}{32.6} & $\pm$0.2 & 9.8 & $\pm$1.1 & 8.7 & $\pm$0.8 & 31.5 & $\pm$0.2 \\

& Walker2D & 0.5 & $\pm$0.3  & 20.3 & $\pm$4.4  & \colorbox{mine}{21.9} & $\pm$1.0 & 18.2  &
$\pm$5.0 & 1.0 & $\pm$1.9 & 20.5 & $\pm$2.3 & \colorbox{mine}{21.7} & $\pm$0.1 &
\colorbox{mine}{21.7} & $\pm$0.1 \\

\midrule
\multirow{3}{*}{\rotatebox[origin=c]{90}{Medium}} 

& HalfCheetah & 57.7 & $\pm$1.5  & \colorbox{mine}{60.6} & $\pm$1.4 & 58.5 & $\pm$2.9 & 59.7 &
$\pm$2.5 & 60.6 & $\pm$2.4 & 60.3 & $\pm$0.8 & \colorbox{mine}{63.7} & $\pm$2.3 & 60.6 & $\pm$1.7 \\

& Hopper & 35.4 & $\pm$4.1  & 41.4 & $\pm$8.8  & 75.3 & $\pm$24.1  & \colorbox{mine}{104.7} &
$\pm$1.0 & 81.3 & $\pm$15.7 & 78.8 & $\pm$19.8 & 102.0 & $\pm$3.8 &
\colorbox{mine}{103.5} & $\pm$0.2\\

& Walker2D & 56.9 & $\pm$25.8  & 58.3 & $\pm$17.9  & 57.3 & $\pm$19.9  & \colorbox{mine}{67.8} &
$\pm$14.4 & 85.3 & $\pm$1.3 & 88.3 & $\pm$1.2 & \colorbox{mine}{88.8} & $\pm$0.9 & 86.5 & $\pm$0.7
\\

\midrule
\multirow{3}{*}{\rotatebox[origin=c]{90}{\shortstack{Medium\\Replay}}} 

& HalfCheetah & 53.5 & $\pm$2.0  & 56.2 & $\pm$1.5 & 57.3 & $\pm$1.2 & \colorbox{mine}{57.5} &
$\pm$0.9 & 55.7 & $\pm$0.9 & \colorbox{mine}{58.9} & $\pm$0.7 & 58.4 & $\pm$0.7 &
\colorbox{mine}{58.9} & $\pm$1.3 \\

& Hopper & 36.0 & $\pm$2.7  & 38.1 & $\pm$4.3  & 42.3 & $\pm$8.4 & \colorbox{mine}{102.0}  &
$\pm$1.2 & 69.0 & $\pm$27.0 & 100.3 & $\pm$3.6 & \colorbox{mine}{102.9} & $\pm$0.5
& 86.2 & $\pm$20.9 \\

& Walker2D & \colorbox{mine}{88.2} & $\pm$5.4  & 75.8 & $\pm$13.6 & 77.9 & $\pm$13.3 & 84.8 &
$\pm$2.5 & 83.1 & $\pm$5.0 & \colorbox{mine}{84.1} & $\pm$1.4 & 82.4 & $\pm$2.9 & 76.8 & $\pm$0.6 \\

\midrule
\multirow{3}{*}{\rotatebox[origin=c]{90}{\shortstack{Medium\\Expert}}} 

& HalfCheetah & \colorbox{mine}{98.0} & $\pm$3.6  & 68.6 & $\pm$17.7 & 86.9 & $\pm$17.4 & 74.3 &
$\pm$16.8 & 95.0 & $\pm$1.7 & \colorbox{mine}{99.5} & $\pm$2.4 & \colorbox{mine}{99.5} & $\pm$1.9 &
99.1 & $\pm$2.5 \\

& Hopper & 47.6 & $\pm$8.3  & 56.3 & $\pm$14.9 & 65.6 & $\pm$16.1 & \colorbox{mine}{107.0} &
$\pm$1.2 & 104.5 & $\pm$7.7 & \colorbox{mine}{106.9} & $\pm$3.0 &102.1 & $\pm$12.6
& 93.8 & $\pm$10.4 \\

& Walker2D & 106.2 & $\pm$1.2  & 65.9 & $\pm$35.6  & 83.9 & $\pm$21.3 & \colorbox{mine}{106.8} &
$\pm$5.1 & 107.7 & $\pm$0.8 & \colorbox{mine}{108.4} & $\pm$0.5 & 107.9 & $\pm$0.4 & 93.7 &
$\pm$25.6 \\

\midrule
\multicolumn{1}{l}{} & \textbf{Average} & \multicolumn{2}{c}{52.1} & \multicolumn{2}{c}{48.8}
&\multicolumn{2}{c}{55.8} & \multicolumn{2}{c|}{\colorbox{mine}{70.4}} & \multicolumn{2}{c}{66.8}
&\multicolumn{2}{c}{70.7} & \multicolumn{2}{c}{\colorbox{mine}{72.7}} &
\multicolumn{2}{c}{70.2}\\

% \multicolumn{1}{l}{} & \textbf{Median} & \multicolumn{2}{c}{50.6} & \multicolumn{2}{c}{56.3}
% &\multicolumn{2}{c}{57.9} & \multicolumn{2}{c|}{\colorbox{mine}{71.0}} & \multicolumn{2}{c}{75.2}
% &\multicolumn{2}{c}{81.4} & \multicolumn{2}{c}{\colorbox{mine}{85.6}} &
% \multicolumn{2}{c}{81.5}\\

\multicolumn{1}{l}{} & \textbf{IQM} & \multicolumn{2}{c}{47.9} & \multicolumn{2}{c}{51.8}
&\multicolumn{2}{c}{59.4} & \multicolumn{2}{c|}{\colorbox{mine}{74.4}} & \multicolumn{2}{c}{72.5}
&\multicolumn{2}{c}{78.3} & \multicolumn{2}{c}{\colorbox{mine}{82.5}} &
\multicolumn{2}{c}{77.1}\\

\bottomrule
\end{tabular}
}
\end{table*}

\end{appendices}

% REFERENCES
% For some reason we need to add \typeout{} to compile references in Overleaf
% See: https://www.overleaf.com/learn/latex/Errors/Citation_XXX_on_page_XXX_undefined_on_input_line_XXX
% \typeout{}
\clearpage
\section*{Declarations}
\paragraph{Funding.} Carlos E. Luis, Alessandro G. Bottero, Julia Vinogradska
and Felix Berkenkamp received funding (salary) from Bosch Corporate Research for
conducting this study.

\bibliography{references}

\end{document}